\documentclass[letterpaper,twocolumn,10pt]{article}
\usepackage{usenix}

\usepackage{cite}
\usepackage{float}
\usepackage{amsmath,amssymb,amsfonts,amsthm,nccmath}
\usepackage{bm}
\usepackage[algo2e]{algorithm2e}
\usepackage[dvipsnames]{xcolor}
\usepackage[caption=false,font=footnotesize]{subfig}
\usepackage{multirow}
\usepackage{tabularx,threeparttable}
\usepackage{booktabs,tikz,soul,pifont}
\usepackage{enumitem}
\usepackage{graphicx}
\usepackage[most]{tcolorbox}

\newcommand{\rom}[1]{\expandafter{\romannumeral #1\relax}}

\newtheorem{definition}{Definition}[]
\newtheorem{assumption}[definition]{Assumption}

\newtheorem{theorem}{Theorem}[section]
\newtheorem{lemma}[theorem]{Lemma}
\newtheorem{corollary}[theorem]{Corollary}
\newtheorem{remark}[theorem]{Remark}

\newcommand{\unA}{\bar{\mathcal{A}}}
\newcommand{\thu}{\theta_u}
\newcommand{\Chal}{\mathbf{Chal}}
\newcommand{\Aud}{\mathbf{Aud}}
\newcommand{\Comp}{\mathsf{Comp}}
\newcommand{\Cur}{\mathsf{Cur}}

\newcommand{\Vrfy}{\mathsf{Verify}_\Pi}
\newcommand{\Hth}{H_\theta}
\newcommand{\Hinv}{H_\theta^{-1}}
\newcommand{\Hhalf}{\Hth^{-1/2}}
\newcommand{\TV}{\mathsf{TV}}

\newcommand{\hon}{\mathrm{h}}
\newcommand{\dis}{\mathrm{d}}
\newcommand{\mem}{\mathrm{mem}}
\newcommand{\non}{\mathrm{non}}
\newcommand{\loc}{\mathrm{loc}}
\newcommand{\conv}{\mathrm{conv}}

\DeclareMathOperator*{\argmin}{argmin}

\SetCommentSty{mycommfont}
\RestyleAlgo{ruled}
\newenvironment{game}[1][htb]
{
    
    \begin{algorithm2e}[#1]%
    \DontPrintSemicolon
    \LinesNumbered
    \SetNoFillComment
    \SetInd{0.4em}{0.4em}
    \setlength{\algomargin}{1.5em}
    \linespread{1.2}\selectfont
}{\end{algorithm2e}}

\tcbset{
  groupbox/.style={
    enhanced, sharp corners,
    colback=white, colframe=white, boxrule=0pt,        
    borderline={0.6pt}{0pt}{black, dotted},            
    coltitle=black, colbacktitle=white, fonttitle=\bfseries\small,
    halign title=center,
    attach boxed title to top center={yshift=0mm},    
    boxed title style={colback=white, colframe=white, boxrule=0pt},
    top=0mm, bottom=0mm, left=1mm, right=1mm,
  }
}

\title{Characterizing Privacy-Audit Alignment in\\Behavioral Audit of Machine Unlearning}

\author{
{\rm Liou Tang}\\
University of Pittsburgh
\and
{\rm James Joshi}\\
University of Pittsburgh
\and
{\rm Ashish Kundu}\\
Cisco Research
} 

\begin{document}
\maketitle
\thispagestyle{empty}
\pagestyle{empty}

\begin{abstract}

The removal of learned data from Machine Learning models through Machine Unlearning (MU) has been widely studied; however, there is no agreed-upon scheme for auditing MU. Existing work shows that a dishonest model owner can falsify evidence to avoid executing MU, while curious auditors (and adversaries) can infer privacy-sensitive properties of the model and its training data even with limited access. Yet auditing of MU under mutual distrust between the model owner and the auditor remains unexplored.
In this paper, we characterize how much a generic audit scheme that relies solely on querying the model for \emph{behavioral} signals inevitably results in privacy leakage related to the retained set by providing a \emph{geometric} transfer theorem that establishes a lower bound on retained set membership distinguishability based on the audit accuracy.
In addition, we study how the unlearned set, target sample, and query protocol jointly determine the privacy-audit transfer coefficient through local and global model-parameter-space geometry.
Our empirical experiments on both convex and non-convex models strongly support these results.
Our results call for more careful consideration of the privacy-audit tension under a realistic auditor model and serve as a foundation for greater scrutiny of privacy-preserving audit scheme designs for the MU pipeline.
We also release our \href{https://github.com/LiouTang/Behavioral-Unlearn-Audit}{code implementation}.
\end{abstract}

\section{Introduction}\label{sec:intro}

\emph{Machine Unlearning} (MU) \cite{Cao2015Towards,Bourtoule2021SISA,Nguyen2025Survey} has emerged as a key approach to support removal of private data, copyrighted or intellectual property information used for training Machine Learning (ML) models, as well as to provide a basis for compliance with regulations that support \emph{right to be forgotten/erasure/delete} or \emph{right to rectification}, etc., highlighted by various regulations such as the General Data Protection Regulation (GDPR) \cite{GDPR2016} and the California Consumer Privacy Act (CCPA) \cite{CCPA2018}. A variety of MU algorithms across different model architectures have been developed, marking it an active and rapidly developing field \cite{Nguyen2025Survey,Triantafillou2024Progress,Shaik2025Exploring}

However, the \emph{verification} and \emph{audit} of MU algorithms remains largely under-explored; in particular, in a Machine Learning as-a-service (MLaaS) scenario, a user/client can request the model owner to unlearn their data, but there is often lack of verifiable confirmation of unlearning, while a dishonest/malicious model owner can actively forge evidence of unlearning to retain information of the user data \cite{Thudi2022Necessity,Zhang2024Fragile}. In a worst-case scenario, even the unaltered original model parameters offer ``plausible deniability'' for unlearning \cite{Thudi2022Necessity}. Alternatively, an \emph{auditor} (a third-party such as a governmental agency responsible for regulatory enforcement) that has aggregated knowledge beyond singular user data poses stronger abilities to launch privacy inference attacks comparable to active \emph{adversaries} against ML/MU; therefore, even an honest auditor that strictly adheres to the audit protocol can either inadvertently or intentionally cause privacy leakage on data retained in the model \cite{Shokri2017MIA,Carlini2022MIA,Chen2021Jepardize,Hayes2024Inexact}.

In this paper, we take the first step toward addressing the clear privacy-audit tension that arises when considering an untrustworthy model owner and an honest-but-curious auditor. We make the following claim:
\begin{quote}
    Given a \emph{behavioral} audit scheme for MU that purely relies on query access, characterized in Sec. \ref{ssec:threat-model}, an auditor aims to distinguish whether the unlearned set was honestly or dishonestly unlearned from the target model. At the same time, the \emph{honest-but-curious} auditor aims to infer the membership status of whether a target data sample is in the retained set of the model. We prove that the membership distinguishability of the target sample $z^*$ is aligned with the accuracy of the audit protocol, subject to a transfer (alignment) coefficient $\gamma_Q(z^*)$: sufficiently accurate behavioral compliance auditing forces a lower bound on membership distinguishability for the target.
\end{quote}
We provide a closed-form function of $\gamma_Q$ defined entirely in behavioral space, with no assumption on model convexity or the (un-)learning algorithm.
This privacy leakage can be further characterized with knowledge of the local or global geometry of the model parameter space, with approximation transfer coefficient $\lambda_Q$ and uncertainty budget $U_Q$.
For convex models, we show how data geometry and audit-query choice jointly determine $\lambda_Q$ and, consequently, the privacy leakage.
Our results, supported by empirical experiments on both convex and non-convex models, provide a concrete and needed foundation for examining privacy leakage from auditors in MU and call for caution in future MU audit scheme designs.

\section{Related Works}\label{sec:lit}

\subsection{Machine Unlearning}\label{ssec:lit-mu}

First proposed by Cao et al. in \cite{Cao2015Towards}, Machine Unlearning (MU) aims to efficiently remove the \emph{influence} of training data from an already-trained Machine Learning (ML) model \cite{Cao2015Towards,Bourtoule2021SISA}.
Under this general definition, Cooper et al. in their recent work \cite{Cooper2024Policy} further group MU algorithms into three categories:
\begin{enumerate}[label=(\roman*)]
    \item \textbf{Exact Unlearning}, which (partially) retrains the model to ensure removal of the unlearned set \cite{Bourtoule2021SISA}. Bourtoule et al. in \cite{Bourtoule2021SISA} propose the first exact unlearning algorithm, Sharded, Isolated, Sliced, and Aggregated (SISA) training, by partitioning the training set into disjoint subsets and training ensembled sub-models on the subsets. Retraining in this approach is therefore limited to sub-models whose training sets contain samples to be unlearned. Aldaghri et al. in \cite{Aldaghri2021Coded} and Yan et al. in \cite{Yan2022Arcane} further propose different training set partition strategies to avoid performance loss. In \cite{Chowdhury2025ScaleEMU}, Chowdhury et al. propose a SISA-based exact unlearning scheme on Large Language Models (LLMs), in which different parameter-efficient fine-tuning (PEFT) layers are trained on partitioned subsets, ensuring parameter isolation and exact unlearning through deactivation and retraining on limited layers.

    \item \textbf{Approximate Unlearning}, which updates the model parameters to remove/minimize the {estimated} influence of the unlearned samples \cite{Neel2021Descent,Gupta2021Adaptive,Kurmanji2023Unbounded,Fan2024SalUn,Huang2024Unified}, and has been established as the \emph{de facto} unlearning approach for its efficiency and scalability \cite{Triantafillou2024Progress,Shaik2025Exploring}. Huang et al. in \cite{Huang2024Unified} position (gradient-based) approximate unlearning as a combination of three objectives: minimize the influence of unlearned samples, maintain performance on retained samples, and constrain the magnitude of parameter updates. Kurmanji et al. in \cite{Kurmanji2023Unbounded} provide an alternative information-theoretic definition of maximizing divergence between the unlearned model and the original model on the unlearned set while minimizing divergence on the retained set.

    \item \textbf{Output Filtering} (for generative models), which avoids modifying the model parameters, but instead detects and prevents the generation of undesirable output, especially for LLMs. The filter can be placed at user input- \cite{GPT4Sys}, system prompt- \cite{Pawelczyk2024InContext}, and output-levels \cite{Thaker2024Guardrail}. However, these filters are imprecise and can be circumvented by jailbreak attacks \cite{Liu2025Rethinking,Chao2025Jailbreak}, and fall outside the scope of our paper.
\end{enumerate}

In this paper, we target approximate unlearning (category \rom{2}) and a family of \emph{$(\varepsilon, \delta)$-certified} unlearning algorithms in particular \cite{Guo2020Cert,Koloskova2025Cert}.
To formalize, we define ML and (approximate) MU as follows, following Chourasia and Shah \cite{Chourasia2023True}:

\begin{definition}[Machine Learning \cite{Chourasia2023True}]\label{def:ML}
    Let $\mathcal{Z} = \mathcal{X} \times \mathcal{Y}$ be the data population, in which $\mathcal{X}$ and $\mathcal{Y}$ are the input and output domains, respectively. A learning algorithm:
    \begin{equation}
        \mathcal{A}: \mathcal{Z}^n \to \Theta,
    \end{equation}
    \noindent takes a dataset $D \in \mathcal{Z}^n$ with $n$ samples as the input and produces a model $\theta \in \Theta$, the model is queried with $f_\theta: \mathcal{X} \to \mathcal{Y}$.
\end{definition}

\begin{definition}[Honest Machine Unlearning \cite{Chourasia2023True}]\label{def:MU}
    Given a model $\theta = \mathcal{A}(D)$ (Def. \ref{def:ML}), let $D_u \subseteq D$ be the \emph{unlearned set} with $n_u$ samples  and $D_r = D \setminus D_u$ be the \emph{retained set}. An unlearning algorithm:
    \begin{equation}
        \unA: \mathcal{Z}^n \times \mathcal{Z}^{n_u} \times \Theta \to \Theta,
    \end{equation}
    \noindent produces an updated/unlearned model $\thu$, such that $\thu$ is \emph{$(\varepsilon, \delta)$-indistinguishable} \cite{Dwork2006DP} from a model trained on the retained set $D_r$.
\end{definition}

\subsection{Auditing Machine Unlearning}\label{ssec:lit-audit}

Beyond a general definition of MU (Def. \ref{def:MU}), consider an arbitrary \emph{dishonest} (or \emph{partially honest}) unlearning algorithm $\unA^d$ distinct from an honest counterpart $\unA$, which we define as follows:

\begin{definition}[Honest and Dishonest Unlearning Endpoints]\label{def:comp-endpoints}
    Let the honest endpoint be $\thu^\hon := \unA(\theta,D,D_u)$, and $\thu^\dis := \unA^d(\theta,D,D_u)$ be the endpoint returned by an arbitrary owner strategy $\unA^\dis$ distinct from the honest unlearning algorithm. We call $(\thu^\hon,\thu^\dis)$ the compliance endpoint pair.
\end{definition}

In this scenario, equally important as the model owner's \emph{execution} of MU is the auditor's \emph{verification} of MU, i.e., whether one can distinguish between $\thu^\hon$ and $\thu^\dis$. Thudi et al. in \cite{Thudi2022Necessity} provide the first argument for auditing and verifying MU. As demonstrated by Shumailov et al. in \cite{Shumailov2021Ordering}, two ML models trained on adjacent but different datasets can produce the same parameters by manipulating the order of data batches in gradient calculations, using a data-ordering attack with stochastic gradient descent (SGD). Therefore, the model owner can claim to have executed the MU algorithm even without modifying the model itself. From this, we see two ways to address MU verification/audit schemes:
\begin{enumerate}[label=(\roman*)]
    \item \textbf{Behavioral Audit of Unlearning}, in which the auditor queries the unlearned model for behavioral signals. Sommer et al. in \cite{Sommer2022Athena} and Guo et al. in \cite{Guo2024Verifying} propose injecting triggers into training-set samples and submitting unlearning requests on those samples. Therefore, an honest MU algorithm will result in the absence of backdoor behavior on queries with triggered input. Gao et al. in \cite{Gao2024VeriFi} further extend backdoor-based MU verification to federated (un-)learning.

    More generally, we can include \emph{membership inference attacks} (MIAs) against ML/MU on the unlearned set $D_u$ in this family of auditing schemes, in which honest MU will ensure that any unlearned sample is no longer classified as a member of the training set of $\thu$ \cite{Shokri2017MIA,Carlini2022MIA,Zarifzadeh2024LowCost,Chen2021Jepardize,Gao2022Deletion,Kurmanji2023Unbounded,Hayes2024Inexact,Tang2025Apollo}. Hayes et al. in \cite{Hayes2024Inexact} propose U-LiRA, an adaptation of the likelihood-ratio attack (LiRA) proposed by Carlini et al. \cite{Carlini2022MIA}, and examine the likelihood of a model being trained and subsequently unlearn the target sample $x$ or not trained on $x$ at all. Tang et al. in \cite{Tang2025Apollo} propose a label-only MIA against MU by examining artifacts of Under- and Over-Unlearning in $\thu$, which requires neither access to the original model $\theta$ nor to the posterior probabilities.
    
    Specifically for \emph{audits} on MU, Gu et al. in \cite{Gu2025Audit} propose A-LiRA, a data augmentation-based LiRA attack for auditing MU algorithms, and empirically demonstrate that approximate unlearning algorithms compromise the privacy of both $D_u$ and $D_r$ even when the model is trained with DP guarantees. Sarmasarkar et al. \cite{Sarmasarkar2026Audit} further formalizes using membership distinguishability of the unlearned set $D_u$ over auditor-controlled candidate forget sets to lower-bound the \emph{soundness guarantee}, i.e., $\varepsilon$- or $(\varepsilon, \delta)$-indistinguishability of certified unlearning (Def. \ref{def:MU}). In contrast, we ask whether behavioral evidence collected to audit MU can further expose membership information about \emph{retained} data $D_r$.

    \item \textbf{Reproducible Proof-of-Unlearning}, in which the model owner offers cryptographic guarantees of execution of the unlearning algorithm. Weng et al. in \cite{Weng2024PoU} propose a trusted execution environment (TEE)-based protocol to verify MU. Eisenhofer et al. in \cite{Eisenhofer2025Verifiable} propose a verifiably secure unlearning protocol utilizing Succinct Non-Interactive Arguments of Knowledge (SNARKs) \cite{Setty2020Spartan} and hash chains. More recently, Ye et al. \cite{Ye2026Audit} propose using retained-only verification models to audit unlearned models' proof-of-ignorance.
\end{enumerate}
Zhang et al. in \cite{Zhang2024Fragile} propose two approaches targeting both auditing/verification schemes (\rom{1}) and (\rom{2}) by \emph{forging} a $\{{w}^{(t)}, d^{(t)}, g^{(t)}\}$-triplet proof-of-unlearning/retraining, which tracks the update of model weights, unlearned data, and gradient function. A malicious model owner uses mini-batches from $D_r$ and the original gradient steps when training the model to mimic the behavior of $D_u$, thus providing a strong attack against MU auditing.

In this paper, we present a complementary approach to that of Zhang et al. in \cite{Zhang2024Fragile}; in particular, instead of an empirical refutation of unlearning verification/audits, we study when (and how) behavioral audits of MU (scheme \rom{1}) incur privacy leakage on the retained set by proving a geometric bound for membership distinguishability of the target sample given a behavioral audit protocol. Additionally, our contribution is also complementary to the empirical privacy attack/audits of unlearning in recent works \cite{Hayes2024Inexact,Tang2025Apollo,Gu2025Audit,Sarmasarkar2026Audit}. While these works construct concrete attacks that demonstrate privacy leakage of the unlearned data in specific unlearning algorithms, we focus on establishing the underlying geometric transfer theorem that explains privacy leakage of the retained data as an intrinsic property of the auditing process.

Following Carlini et al. \cite{Carlini2022MIA}, we can define an audit protocol $\Pi$ towards MU as a \emph{game} $\Comp^{\Pi}$ (Game \ref{game:comp}) between the \emph{challenger} (model owner) and the \emph{auditor} that tests whether or not the unlearned model $\thu$ complies with honest unlearning. We also study an \emph{honest-but-curious} auditor who aims to infer whether a target sample $z^*$ is a member of the retained set $D_r$. We present this as the curiosity game $\Cur^{\Pi}$ (Game \ref{game:cur}), \emph{in parallel with} the compliance game.

\begin{game}[t]
\caption[F]{Compliance game $\Comp^{\Pi}$}\label{game:comp}
    \KwIn{$\thu^\dis$, $\thu^\hon$, bit $b \in \{0, 1\}$}
    \uIf{$b = 0$}{
        $\thu \gets \thu^\dis$
        \tcp*[l]{dishonest/partial endpoint}
    }
    \Else{
        $\thu \gets \thu^\hon$
        \tcp*[l]{honest endpoint}
    }
    $\tau \gets \mathsf{Audit\_Transcript}_\Pi(\thu)$ \;
    $\hat{b} \gets \Vrfy (\tau)$ \;
    \KwOut{$\mathbf{1}[\hat{b} = b]$}
\end{game}
 
\begin{game}[t]
\caption[F]{Curiosity game $\Cur^{\Pi}$}\label{game:cur}
    \KwIn{$\thu^\mem$, $\thu^\non$, bit $m \in \{0,1\}$}
    \uIf{$m = 1$}{
        $\thu \gets \thu^\mem$
        \tcp*[l]{$z^* \in D_r$}
    }
    \Else{
        $\thu \gets \thu^\non$
        \tcp*[l]{$z^* \notin D_r$}
    }
    $\tau \gets \mathsf{Audit\_Transcript}_\Pi(\thu)$ \;
    $\hat{m} \gets \mathsf{Infer}(\tau, D_u, z^*)$ \;
    \KwOut{$\mathbf{1}[\hat{m}=m]$}
\end{game}

Based on Games \ref{game:comp} and \ref{game:cur}, to account for the dishonest model owner and the honest-but-curious auditor, we propose a probability-based definition of accountability for $\Pi$:

\begin{definition}[$(\alpha, \alpha^\prime, \beta)$-Accountability]\label{def:accountability}
    For the sensitive property $s$, given $\alpha, \alpha^\prime, \beta \in [0,1]$, we say an audit protocol $\Pi$ is $(\alpha, \alpha^\prime, \beta)$-accountable on $s$ if it satisfies the following:
    \begin{itemize}
        \item \textbf{Soundness:} under the compliance game (Game \ref{game:comp}), for any $\thu$ produced by $\unA^\dis$, $\Pi$ ensures:
        \begin{equation}
            \Pr \left[\Vrfy(\tau) = 1 \mid \thu = \thu^\dis \right] \leq \alpha,
        \end{equation}
        where the false negative rate of misclassifying dishonest unlearning as honest is bounded by $\alpha$.

        \item \textbf{Completeness:} under the compliance game (Game \ref{game:comp}), for any $\thu$ produced by $\unA^\hon$, $\Pi$ ensures:
        \begin{equation}
            \Pr \left[\Vrfy(\tau) = 0 \mid \thu = \thu^\hon \right] \leq \alpha^\prime,
        \end{equation}
        where the false positive rate of misclassifying honest unlearning as dishonest is bounded by $\alpha^\prime$

        \item \textbf{Privacy:} under the curiosity game (Game \ref{game:cur}), let $P_\tau^{\mem}$ and $P_\tau^{\non}$ be the two transcript distributions in Game \ref{game:cur}, $\Pi$ ensures:
        \begin{equation}\label{eq:cur-adv}
            \beta(z^*) := \TV \left(P_\tau^{\mem},P_\tau^{\non}\right) \leq \beta,
        \end{equation}
        that is, the transcript membership \emph{distinguishability} of the target sample $\beta(z^*)$, given by the \emph{total variance} ($\TV$) between $P_\tau^{\mem}$ and $P_\tau^{\non}$, is bounded by $\beta$\footnote{We again stress an affinity between this definition and the LiRA attack of Carlini et al. \cite{Carlini2022MIA} as well as the U-LiRA attack of Hayes et al. \cite{Hayes2024Inexact}. We expand on this intuition in Sec. \ref{ssec:bridge}.}. The Bayes-optimal membership attack success probability under equal priors is $(1+\beta(z^*))/2$ \cite{Shokri2017MIA}.
    \end{itemize}
\end{definition}

\section{Behavioral Auditing of Machine Unlearning}\label{sec:behavioral-audit}

\subsection{Threat Model}\label{ssec:threat-model}

\paragraph{Goals} We study the model owner-auditor game under a dishonest challenger (model owner) $\Chal$ and an honest-but-curious auditor $\Aud$, both parties have distinct goals: (\rom{1}) $\Chal$ aims to release a dishonestly/partially honestly unlearned model $\thu^\dis$ that passes the audit; and (\rom{2}) $\Aud$ adheres strictly to the audit protocol $\Pi$, however, it also aims to infer the membership status of target samples through knowledge gained during the query. \emph{We do not assume any specific model architecture or learning/unlearning mechanism}.

\paragraph{Access} We assume that $\Aud$ has oracle access to $\thu$ and can submit arbitrary query $x_t$ to $\thu$ up to a query budget $T$. Further, we follow the common assumption that $\Aud$ can draw \emph{surrogate} datasets from the same underlying data distribution $\mathcal{Z}$. $\Chal$ releases an unlearned model $\thu$, and allows $\Aud$ information on its unlearned set $D_u$, training algorithm $\mathcal{A}$ and the (honest) unlearning algorithm $\unA$.

\paragraph{Assumptions} In this paper, we focus on an $\Aud$ that is honest-but-curious and hence can infer \emph{membership} of the retained set sample\footnote{While we only study the case for membership inference in this paper, we can modify the curiosity game $\Cur^{\Pi}$ to infer an arbitrary privacy-sensitive property of $D_r$. We discuss an extension to more general privacy inferences in Sec. \ref{ssec:dis-priv-inf}.}. The curiosity advantage $\beta$ under this construction is a \emph{probability} metric, which can be interpreted as the advantage of $\Aud$ with access to $\tau$ over random guessing.
We define a \emph{purely behavioral} audit scheme as follows:

\begin{assumption}[Fixed-Query Behavioral Audit]\label{assum:audit}
    A behavioral audit proceeds as follows: the auditor $\Aud$ follows the audit protocol $\Pi$, which fixes a deterministic query sequence $Q=(x_1,\ldots,x_T)$ dependant only on $D_u$ and auditor knowledge (Sec. \ref{ssec:threat-model}). At each round, $\Aud$ submits $x_t$ to $\thu$, and receives oracle response of the posterior probability $f_{\thu}(x_t)$; this proceeds for $T$ rounds. The resulting audit transcript is a sequence:
    \begin{equation}
        \tau = \{(x_t, f_{\thu}(x_t) + \mathcal{N}(0, \sigma^2))\}_{t=1}^T,
    \end{equation}
    Importantly, we assume that the query response is observed with additive Gaussian noise with variance $\sigma^2$ as a proxy for stochastic errors from both $\Chal$ and $\Aud$, e.g., naturally occurring model variations in training. $\Chal$ reveals no parameter- or gradient-level information to $\Aud$.
\end{assumption}

\subsection{Bridging Compliance and Curiosity}\label{ssec:bridge}

Consider an ideal auditor $\Aud$ described in Sec. \ref{ssec:threat-model} that is able to train and unlearn arbitrary surrogate models on the training sets sampled from $\mathcal{Z}$. When querying the target model on $x$, such an $\Aud$ captures two signals:
\begin{itemize}
    \item \textbf{The compliance signal:} the discrepancy between the posterior probability on $x$ of a dishonestly unlearned model with an honest baseline.

    \item \textbf{The curiosity signal:} the discrepancy between the posterior probability on $x$ of an unlearned model with $z^* \in D_r$ or with $z^* \notin D_r$.
\end{itemize}

\noindent The auditor repeats this process for all queries in the sequence $Q$. We extend these signals to the entire query sequence:

\begin{definition}[Compliance and Curiosity Vectors]\label{def:sig-vectors}
    For a fixed query sequence $Q=(x_1,\dots,x_T)$, define the behavioral map:
    \begin{equation}
        F_Q(\vartheta) := (f_\vartheta(x_1), \dots, f_\vartheta(x_T)) \in \mathbb{R}^T.
    \end{equation}
    Its compliance and curiosity vectors are, respectively:
    \begin{equation}\label{eq:query-vectors}
    \begin{aligned}
        c_Q &:= F_Q(\thu^\hon)-F_Q(\thu^\dis),\\
        p_Q(z^*) &:= F_Q(\thu^\mem)-F_Q(\thu^\non).
    \end{aligned}
    \end{equation}
\end{definition}

Additionally, when $c_Q \neq 0$, we further define the exact audit-visible alignment coefficient $\gamma_Q$ on target $z^*$ as:
\begin{equation}\label{eq:gamma-q}
    \gamma_Q(z^*) := \frac{|\langle p_Q(z^*),c_Q\rangle|}{\|c_Q\|_2^2}.
\end{equation}
We decompose $p_Q(z^*)$ as:
\begin{equation}\label{eq:exact-decomposition}
    p_Q(z^*)=a_Q c_Q+p_Q^\perp,\quad
    a_Q:=\frac{\langle p_Q(z^*),c_Q\rangle}{\|c_Q\|_2^2},\quad
    p_Q^\perp\perp c_Q,
\end{equation}
which gives:
\begin{equation}\label{eq:exact-projection-bound}
    \|p_Q(z^*)\|_2\geq \gamma_Q(z^*)\|c_Q\|_2.
\end{equation}
$\gamma_Q$ therefore measures the magnitude of the privacy displacement that lies in the compliance behavioral direction. Def. \ref{def:sig-vectors} and Eq. \ref{eq:gamma-q} then enable us to establish a relationship, as captured in Theorem \ref{th:main}, between the audit accuracy and privacy leakage:

\begin{theorem}[Geometric Privacy-Audit Transfer]\label{th:main}
    Let $(\thu^\hon, \thu^\dis, \thu^\mem, \thu^\non)$ be the  endpoint quadruple and $Q$ be the query sequence (Assumption \ref{assum:audit}) following an $(\alpha, \alpha^\prime, \beta)$-accountable audit protocol $\Pi$ (Def. \ref{def:accountability}). Let $c_Q \neq 0$, then by Eq. \ref{eq:exact-decomposition}, the membership distinguishability on $z^*$ inherent in the transcript $\tau$ is as follows:
    \begin{equation}\label{eq:beta-main-exact}
        \beta(z^*) = 2\Phi \left(
        \frac{1}{2\sigma} \sqrt{\gamma_Q(z^*)^2\|c_Q\|_2^2+\|p_Q^\perp\|_2^2}
        \right) - 1,
    \end{equation}
    therefore:
    \begin{equation}\label{eq:bata-main-lb}
        \beta(z^*) \geq 2 \Phi(\gamma_Q(z^*) q_e) - 1.
    \end{equation}
    in which $e := \alpha + \alpha^\prime \leq 1$, $q_e := \Phi^{-1}(1-e/2)$. A simpler bound is therefore:
    \begin{equation}\label{eq:beta-exact-lb-linear}
        \beta(z^*) \geq \min\{\gamma_Q(z^*),1\}(1-e).
    \end{equation}
\end{theorem}
\begin{proof}
    Define:
    \begin{equation}
        d_{\Comp}:=\|c_Q\|_2, \quad d_{\Cur}:=\|p_Q(z^*)\|_2.
    \end{equation}
    Under Assumption \ref{assum:audit}, the compliance and curiosity transcript distributions are both equal-covariance Gaussians; therefore:
    \begin{equation}\label{eq:alpha-beta-d}
    \begin{aligned}
        \TV(P_\tau^\hon,P_\tau^\dis) &= 2\Phi \left(\frac{d_{\Comp}}{2\sigma}\right) - 1,\\
        \beta(z^*) &= 2\Phi \left(\frac{d_{\Cur}}{2\sigma}\right) - 1. \\
    \end{aligned}
    \end{equation}
    in which:
    \begin{equation}
        d_{\Cur}^2=\gamma_Q(z^*)^2d_{\Comp}^2+\|p_Q^\perp\|_2^2,
    \end{equation}
    which gives Eq. \ref{eq:beta-main-exact}.
    
    \item By Def. \ref{def:accountability}, soundness and completeness imply that:
    \begin{equation}
        \TV(P_\tau^\hon,P_\tau^\dis) \geq P_\tau^\hon(A)-P_\tau^\dis(A) \geq 1-e,
    \end{equation}
    therefore $d_{\Comp} \geq 2\sigma q_e$. Dropping the nonnegative orthogonal term gives $d_{\Cur} \geq \gamma_Q d_{\Comp} \geq 2\sigma\gamma_Q q_e$. Substituting into Eq. \ref{eq:alpha-beta-d} proves Eq. \ref{eq:bata-main-lb}.

    \item For Eq. \ref{eq:beta-exact-lb-linear}, let $g(u):=2\Phi(u)-1$ for $u \geq 0$. Then we have:
    \begin{equation}
    \left\{
    \begin{array}{lll}
        g(\gamma_Qq_e) \geq \gamma_Q g(q_e), \quad &\gamma_Q \leq 1 &(\text{concavity}) \\
        g(\gamma_Qq_e) \geq g(q_e), \quad &\gamma_Q \geq 1 &(\text{monotonicity}) \\
    \end{array}
    \right.
    \end{equation}
    Recall that $g(q_e) = 1-e$; hence Eq. \ref{eq:beta-exact-lb-linear} follows.
\end{proof}

Theorem \ref{th:main} therefore gives an exact characterization of the privacy leakage by the honest-but-curious auditor conditioned on the compliance vector. We also make the following remark:

\begin{remark}[Orthogonalization, Genericity and Residual Privacy]\label{rem:genericity}
    By Eq. \ref{eq:gamma-q}, $\gamma_Q(z^*)=0$ when $\langle p_Q,c_Q\rangle=0$, in which case audit accuracy does not force a lower bound on the curiosity privacy leakage on the compliance direction.
    Intuitively, an auditor $\Aud$ might enforce orthogonalization of $p_Q \perp c_Q$ to ensure that curiosity leakage cannot occur on the compliance signal direction. However, we note two objection to this intuition:
    \begin{enumerate}
        \item By Assumption \ref{assum:audit}, the query sequence $Q$ is selected using only $D_u$ and is independent of $D_r$, while calculating the privacy displacement $p_Q(z^*)$ requires knowledge of $z^*$ and by extension $D_r$ itself.
        
        Let $\mathcal{P}_Q(D_u,D_r)$ denote the set of privacy displacement vectors for all $z^* \in D_r$. A robust transfer-separation guarantee would require $\langle p,c_Q\rangle \approx 0$ for \emph{any} $p \in \mathcal{P}_Q(D_u,D_r)$. If $c_Q$ lies in $\operatorname{span}(\mathcal{P}_Q(D_u,D_r))$, exact orthogonality to every such privacy direction is impossible unless $c_Q=0$.

        \item Further, $\langle p_Q,c_Q\rangle=0$ is \emph{not} in itself a privacy guarantee: as demonstrated by Eq. \ref{eq:beta-main-exact}, an orthogonal component $\|p_Q^\perp\|_2^2 \geq 0$ still yields membership distinguishability from the same transcript $\tau$ even when $\gamma_Q = 0$. A privacy guarantee on transcript-level membership leakage would require $\|p_Q\|_2 \to 0$, or an independent privacy mechanism that limits distinguishability.
    \end{enumerate}
\end{remark}

A robust privacy guarantee under Theorem \ref{th:main} therefore requires additional knowledge of the retained set to bound the privacy-sensitive directions. While the curiosity game (Game \ref{game:cur}) itself restricts access to $D_r$ for $\Aud$ as knowing $D_r$ a priori implies $\Aud$ trivially wins the game, auditor knowledge of the underlying data distribution, model architecture, or the (un-)learning pipeline can be used to constrain the query sequence $Q$.
We expand on this observation in Sec. \ref{ssec:certification} and \ref{ssec:mechanism}, in which we discuss auditor uncertainty on the compliance/curiosity vectors.

\subsection{Approximate Model Geometric Certification}\label{ssec:certification}

As discussed in Sec. \ref{ssec:bridge}, the privacy leakage by the audit protocol $\Pi$ demonstrated in Theorem \ref{th:main} describes the \emph{information-theoretic} privacy leakage intrinsic in the audit transcript $\tau$, it is not a statement on whether it's easily computable or operationally obtainable by an $\Aud$.
In this section, we study when the endpoint quadruple $(\thu^\hon, \thu^\dis, \thu^\mem, \thu^\non)$ is unavailable or not computable (especially in high-dimensional non-convex models), i.e., \emph{how can an we analyze an audit protocol without an exact geometry}.
We begin by introducing an \emph{approximation} of the audit geometry:

\begin{definition}[Approximate Audit-Visible Geometric Certificate]\label{def:geo-cert}
    A tuple $(\tilde{c}_Q, \tilde{p}_Q, \varepsilon_c, \varepsilon_p)$ is a geometric certificate for an endpoint quadruple $(\thu^\hon, \thu^\dis, \thu^\mem, \thu^\non)$ if:
    \begin{equation}\label{eq:certificate-errors}
        \|c_Q-\tilde{c}_Q\|_2\leq\varepsilon_c, \quad \|p_Q-\tilde{p}_Q\|_2\leq\varepsilon_p.
    \end{equation}
    in which $\tilde{c}_Q, \tilde{p}_Q \in \mathbb{R}^T$, $\varepsilon_c,\varepsilon_p\geq0$.
    When $\tilde{c}_Q \neq 0$, we further define its certified transfer coefficient and certification error as:
    \begin{equation}\label{eq:lambda-q}
        \lambda_Q:= \frac{|\langle\tilde{p}_Q,\tilde{c}_Q\rangle|}{\|\tilde{c}_Q\|_2^2},\quad
        U_Q:=\lambda_Q\varepsilon_c+\varepsilon_p.
    \end{equation}
\end{definition}

Let $\tilde{\vartheta}_Q$ be the angle between $\tilde{p}_Q$ and $\tilde{c}_Q$, then:
\begin{equation}\label{eq:lambda-interpret}
    \lambda_Q=\frac{\|\tilde{p}_Q\|_2}{\|\tilde{c}_Q\|_2} | \cos \tilde{\vartheta}_Q |.
\end{equation}
Recall Eq. \ref{eq:gamma-q} and \ref{eq:exact-decomposition}, $\lambda_Q$ and $U_q$ together allow us to introduce \emph{uncertainty} in Theorem \ref{th:main}: $\lambda_Q$ describes the privacy component predicted by the proxy geometry relative to the proxy compliance signal, while $U_Q$ describes the certification error. A large $U_Q$ implies that the available geometric approximation cannot support a strong conclusion of privacy leakage. We can extend Theorem \ref{th:main} as the following lemma:

\begin{lemma}[Certified Privacy-Audit Transfer]\label{lemma:cert}
    Similar to Theorem \ref{th:main}, we replace the direct knowledge of the exact endpoint quadruple with any certificate $(\tilde{c}_Q, \tilde{p}_Q, \varepsilon_c, \varepsilon_p)$ satisfying Def. \ref{def:geo-cert} with $\tilde{c}_Q \neq 0$. Then we have:\begin{equation}\label{eq:main-exact-cert}
        \beta(z^*) \geq 2\Phi \left( \left[\lambda_Q q_e-\frac{U_Q}{2\sigma}\right]_+ \right) - 1.
    \end{equation}
    in which $e$ and $q_e$ are given by Theorem \ref{th:main}. Further:
    \begin{equation}\label{eq:main-linear-cert}
        \beta(z^*) \geq \left[ \min\{\lambda_Q,1\}(1-e)- \frac{U_Q}{\sigma\sqrt{2\pi}} \right]_+.
    \end{equation}
\end{lemma}
\begin{proof}
    The result follows by lower-bounding $\|p_Q\|_2$ and applying the compliance separation $\|c_Q\|_2\ge 2\sigma q_e$. We leave a detailed proof in Appendix \ref{app:proof:cert}.
\end{proof}

From Lemma \ref{lemma:cert}, we immediately arrive at the following corollary:

\begin{corollary}[Uncertainty Threshold]\label{cor:uncertainty-threshold}
    For $\lambda_Q>0$, the lower bound in Eq. \ref{eq:main-exact-cert} is strictly positive whenever:
    \begin{equation}\label{eq:uncertainty-threshold}
        \lambda_Q q_e>\frac{U_Q}{2\sigma},
    \end{equation}
    or equivalently:
    \begin{equation}
        e < 2\left[ 1-\Phi \left( \frac{U_Q}{2\sigma\lambda_Q} \right) \right].
    \end{equation}
\end{corollary}
\noindent The result implies that a high audit accuracy does not necessarily induce privacy leakage on the retained set $D_r$. We can only conclude that the behavioral audit itself incurs privacy leakage when the compliance separation exceeds the uncertainty term. If $\lambda_Q=0$, or if the uncertainty dominates this component, Lemma \ref{lemma:cert} is vacuous.

\subsection{Constructing Approximate Audit-Visible Certificates}\label{ssec:mechanism}

In Sec. \ref{ssec:certification}, we demonstrated that approximation of the model geometry yields a similar bound on privacy leakage through behavioral audit with explicit uncertainty. In this section, we further discuss two mechanisms to provide the approximation: (\rom{1}) a locally differentiable behavioral map applicable to arbitrarily differentiable models, and (\rom{2}) a strongly convex model with a first-order deletion model that exposes a more interpretable data- and query-dependent geometry.

\paragraph{Local endpoint geometry} We start with the assumption of a \emph{local} geometry, where the quadruple $(\thu^\hon, \thu^\dis, \thu^\mem, \thu^\non)$ are within the neighborhood of a reference parameter $\theta_0$, such that $F_Q$ is differentiable near all four endpoints. Define:
\begin{equation}
    J_Q:=\nabla_\theta F_Q(\theta_0), \quad G_Q:=J_Q^\top J_Q,
\end{equation}
in which $J_Q$ is the the Jacobian of the behavioral map $F_Q$ at the reference model parameter $\theta_0$, and let:
\begin{equation}
    \delta_c:=\thu^\hon-\thu^\dis, \quad \delta_p:=\thu^\mem-\thu^\non,
\end{equation}
be the parameter displacement for the honesty/dishonesty and member/non-member endpoint pair, respectively. The Jacobian maps these parameter-space endpoint changes into proxy behavioral directions.

\begin{corollary}[Local Audit Geometry]\label{cor:loc-aud-geo}
    Assume that $\nabla_\theta F_Q$ is $L_Q$-Lipschitz on a convex neighborhood containing $\theta_0$ and the line segments from $\theta_0$ to all four endpoints. Let:
    \begin{equation}\label{eq:local-proxies}
        \tilde{c}_Q:=J_Q\delta_c, \quad \tilde{p}_Q:=J_Q\delta_p.
    \end{equation}
    Eq. \ref{eq:certificate-errors} holds with:
    \begin{equation}\label{eq:local-eps}
    \begin{aligned}
        \varepsilon_c^\loc &:= \frac{L_Q}{2} \left(
        \|\thu^\hon - \theta_0\|_2^2+ \|\thu^\dis - \theta_0\|_2^2
        \right),\\
        \varepsilon_p^\loc &:= \frac{L_Q}{2} \left(
        \|\thu^\mem - \theta_0\|_2^2 + \|\thu^\non - \theta_0\|_2^2
        \right).
    \end{aligned}
    \end{equation}
    If $J_Q\delta_c \neq 0$, the corresponding local transfer coefficients are given by:
    \begin{equation}\label{eq:lambda-loc}
        \lambda_Q^{\mathrm{loc}} =\frac{|\delta_p^\top G_Q\delta_c|}
        {\delta_c^\top G_Q\delta_c}, \quad
        U_Q=\lambda_Q^\loc\varepsilon_c^\loc+\varepsilon_p^\loc.
    \end{equation}
\end{corollary}
\begin{proof}
    We leave a detailed proof in Appendix \ref{app:proof:loc-aud-geo}.
\end{proof}

Eq. \ref{eq:lambda-loc} shows that the relevant parameter-space geometry is induced by the audit through $G_Q = J_Q^\top J_Q$. An endpoint component orthogonal in ordinary Euclidean or optimization geometry need not remain orthogonal after passing through the behavioral interface.

\paragraph{Influence-based construction} Beyond a local geometry of $(\thu^\hon, \thu^\dis, \thu^\mem, \thu^\non)$, we now discuss a more idealized \emph{global} geometry in which we assume model \emph{convexity} on its parameter. These constraints enable us to express both $\delta_c$ and $\delta_p$ with a query-independent inverse Hessian, which allows us to \emph{instantiate} the geometry of Theorem \ref{th:main} by expressing the abstract endpoint directions as data-dependent quantities through the (whitened) gradient at $z^*$ and $D_u$, as well as the honesty level of the model owner $\Chal$. We make the following assumptions of the (ideal) honest/dis-honest unlearning algorithms:

\begin{assumption}[Learning and Unlearning for Convex Models]\label{assum:convex}
    For a convex model $\theta \in \Theta = \mathbb{R}^p$, let $\ell$ be a per-sample loss that is twice continuously differentiable in $\theta$. The trained model is the unique minimizer:
    \begin{equation}\label{eq:convex-fit}
        \theta =
        \argmin_{w \in \mathbb{R}^p} \frac{1}{n} \sum_{z \in D} \ell(w, z) + \frac{\lambda}{2}\|w\|^2,
        \quad \lambda > 0.
    \end{equation}
    The regularized empirical risk is $\mu$-strongly convex with $\mu \geq \lambda$. The Hessian:
    \begin{equation}
        \Hth := \frac{1}{n}\sum_{z \in D} \nabla^2_\theta\ell(\theta, z) + \lambda I \succ 0,
    \end{equation}
    is positive definite. Denote the gradients of a sample $z = (x,y)$ and of the unlearned set $D_u$ as:
    \begin{equation}
        g_z := \nabla_\theta \ell(\theta, z), \quad G_{D_u} := \sum_{z \in D_u} g_z,
    \end{equation}
    then a MU algorithm through influence approximation of $\mathcal{A}(D \setminus D_u)$ is given by:
    \begin{equation}
        \unA(\theta, D, D_u) := \theta + \frac{1}{n} \Hinv G_{D_u},
    \end{equation}
\end{assumption}

For the unlearning algorithm, consider a model owner that aims to preserve the partial influence of the unlearned set in the updated model \cite{Thudi2022Necessity,Zhang2024Fragile}, we give an instance of Def. \ref{def:comp-endpoints} as follows:
\begin{definition}[$\eta$-Honest Unlearning]\label{def:MU-eta}
    For the unlearning algorithm $\unA$ (Def. \ref{def:MU}), let $\Delta := \unA(\theta, D, D_u) - \theta$ denote its update. For $\eta \geq 0$, we define an $\eta$-{honest} unlearning algorithm\footnote{This is far from a typical dishonest unlearning algorithm, as we will demonstrate in Sec. \ref{ssec:setup} and \ref{ssec:result-non-convex}, but its simplicity allows us to investigate the compliance-curiosity geometry in controlled convex settings more directly.} $\unA_\eta$ as:
    \begin{equation}
        \unA_\eta(\theta, D, D_u) := \theta + \eta \cdot \Delta,
    \end{equation}
    Under this definition, $\unA_1$ is the honest unlearning algorithm described by Def. \ref{def:MU}; $\unA_0$ is a ``lazy'' unlearning algorithm, i.e., $\theta = \thu = \unA_0 (\theta, D, D_u)$ \cite{Shumailov2021Ordering,Thudi2022Necessity}.
    The compliance displacement is given by:
    \begin{equation}\label{eq:convex-delta-c}
        \delta_c = \thu^\hon-\thu^\dis = \frac{1-\eta}{n}\Hinv G_{D_u},
    \end{equation}
    while the curiosity displacement is given by:
    \begin{equation}\label{eq:convex-delta-p}
        \delta_p:=\thu^\mem-\thu^\non = -\frac{1}{n}\Hinv g_{z^*}+r_\theta, \quad
        \|r_\theta\|_2\leq R_\theta,
    \end{equation}
    as a first-order leave-one-out perturbation in model parameter, as demonstrated by Koh and Liang in \cite{KohLiang2017InfFunc}. An explicit finite-sample bound of Eq. \ref{eq:convex-delta-p} is proved by Giordano et al. in \cite{Giordano2019Infinitesimal} (Theorem 1).
\end{definition}

Additionally, we note that Def. \ref{def:MU-eta} includes two regimes of unlearning described by Tang et al. in \cite{Tang2025Apollo}, in which:
\begin{enumerate}
   \item $\eta \in [0,1)$ applies \emph{Under-Unlearning}, in which insufficient unlearning preserves information of the unlearned set in the decision boundaries of $\thu$;
   \item $\eta \in (1, +\infty)$ applies \emph{Over-Unlearning}, in which overzealous unlearning imposes performance loss on the retained set by removing too much information.
\end{enumerate}

We define the \emph{whitened gradients} of $z^*$ and $D_u$, respectively, as:
\begin{equation}
    \bar{g}:= \Hhalf g_{z^*}, \quad \bar{G}:= \Hhalf G_{D_u},
\end{equation}
The Hessian \emph{alignment} between the target sample and unlearned set gradients can therefore be written as:
\begin{equation}\label{eq:rho-align}
    \rho(z^*) := \frac{\bar{g}^\top \bar{G}}{ \|\bar{g}\|_2 \|\bar{G}\|_2 }\in[-1,1].
\end{equation}
Projecting the leading curiosity displacement $\delta_p^0 := -(1/n)\Hinv g_{z^*}$ onto $\delta_c$ in the Hessian metric gives the query-independent mechanism's coefficient:
\begin{equation}\label{eq:kappa}
    \kappa(z^*) := \frac{|\rho(z^*)|}{|1-\eta|} \cdot \frac{\|\bar{g}\|_2}{\|\bar{G}\|_2}, \quad \eta \neq 1.
\end{equation}
$\rho$ and $\kappa$ therefore describes the relative geometry between $z^*$ and $D_u$ before the behavioral audit $\Pi$. Importantly, it does not imply an upper/lower bound on the privacy-audit alignment coefficient $\gamma_Q$, as the query sequence $Q$ changes the metric in which the compliance/curiosity directions are observed. 

To make the query dependence explicit, take the (local) reference parameter $\theta_0=\theta$; let $J_Q:=\nabla_\theta F_Q(\theta)$, and $G_Q:=J_Q^\top J_Q$. We define the Hessian-whitened query metric:
\begin{equation}\label{eq:MQ}
\begin{aligned}
    M_Q :=& \Hhalf G_Q \Hhalf \\
    =& (J_Q \Hhalf )^\top(J_Q \Hhalf ) \succeq 0.
\end{aligned}
\end{equation}
Equivalently, we write $h_{x_t} := \Hhalf \nabla_\theta f_\theta(x_t)$, then $M_Q=\sum_{t=1}^T h_{x_t}h_{x_t}^\top$. $M_Q$ is the metric imposed by the audit queries on Hessian-whitened parameter directions. Let:
\begin{equation}\label{eq:convex-query-proxies}
\begin{aligned}
    \tilde{c}_Q &:= J_Q\delta_c = \frac{1-\eta}{n}J_Q \Hhalf \bar{G},\\
    \tilde{p}_Q &:=J_Q\delta_p^0 = -\frac{1}{n}J_Q \Hhalf \bar{g}.
\end{aligned}
\end{equation}
For the semidefinite norm $\|v\|_{M_Q}:=\sqrt{v^\top M_Qv}$, Eq. \ref{eq:convex-query-proxies} gives:
\begin{equation}\label{eq:proxy-magnitudes}
    \|\tilde{c}_Q\|_2=\frac{|1-\eta|}{n}\|\bar{G}\|_{M_Q}, \quad
    \|\tilde{p}_Q\|_2=\frac{1}{n}\|\bar{g}\|_{M_Q}.
\end{equation}
The query-conditioned Hessian alignment is:
\begin{equation}\label{eq:rho-q}
    \rho_Q(z^*):= \frac{ \bar{g}^\top M_Q\bar{G} }{ \|\bar{g}\|_{M_Q} \|\bar{G}\|_{M_Q} }.
\end{equation}
whenever $\|\bar{g}\|_{M_Q} > 0$ and $\|\bar{G}\|_{M_Q} > 0$, with the corresponding first-order audit-visible transfer coefficient:
\begin{equation}\label{eq:lambda-conv}
    \lambda_Q^\conv
    =\frac{|\langle\tilde{p}_Q,\tilde{c}_Q\rangle|}{\|\tilde{c}_Q\|_2^2}
    =\frac{|\rho_Q(z^*)|}{|1-\eta|} \cdot \frac{\|\bar{g}\|_{M_Q}}{\|\bar{G}\|_{M_Q}},
\end{equation}
which is the query-conditioned counterpart of $\kappa$. $\tilde{p}_Q^\perp$; the Euclidean component of $\tilde{p}_Q$ orthogonal to $\tilde{c}_Q$, satisfies:
\begin{equation}\label{eq:convex-orthogonal}
    \|\tilde{p}_Q^\perp\|_2 = \frac{1}{n}\|\bar{g}\|_{M_Q}\sqrt{1-\rho_Q(z^*)^2}.
\end{equation}
The convex model therefore provides first-order predictions for both $c_Q$ and $p_Q$ in Eq. \ref{eq:exact-decomposition}, with which we arrive at the following corollary:

\begin{corollary}[Convex Audit Geometry]\label{cor:conv-aud-geo}
    Under Assumption \ref{assum:convex} and Def. \ref{def:MU-eta}, let $\eta \neq 1$, assume $\nabla_\theta F_Q$ is $L_Q$-Lipschitz, use $\tilde c_Q$ and $\tilde p_Q$ from Eq. \ref{eq:convex-query-proxies}. Def. \ref{def:geo-cert} holds with:
    \begin{equation}\label{eq:convex-U}
    \begin{aligned}
        \lambda_Q^\conv &= \frac{|\rho_Q(z^*)|}{|1-\eta|} \cdot \frac{\|\bar{g}\|_{M_Q}}{\|\bar{G}\|_{M_Q}}, \quad
        U_Q^\conv = \lambda_Q^\conv\varepsilon_c^\conv + \varepsilon_p^\conv,
    \end{aligned}
    \end{equation}
    in which:
    \begin{equation}
    \begin{aligned}
        \varepsilon_c^\conv
        &:=\frac{L_Q}{2}(1+\eta^2)\|\Delta\|_2^2,\\
        \varepsilon_p^\conv
        &:=\|J_Q\|_{\mathrm{op}}R_\theta +
        \frac{L_Q}{2}\big( \|\thu^\mem-\theta\|_2^2+\|\thu^\non-\theta\|_2^2 \big).
    \end{aligned}
    \end{equation}
    Lemma \ref{lemma:cert} gives:
    \begin{equation}\label{eq:main-convex}
        \beta(z^*) \geq 2\Phi\left(
        \left[ \lambda_Q^\conv q_e - \frac{U_Q^\conv}{2\sigma} \right]_+
        \right) - 1,
    \end{equation}
    with the corresponding linear bound:
    \begin{equation}\label{eq:main-convex-linear}
        \beta(z^*) \geq
        \left[ \min\{\lambda_Q^\conv,1\}(1-e) - \frac{U_Q^\conv}{\sigma\sqrt{2\pi}} \right]_+.
    \end{equation}
\end{corollary}
\begin{proof}
    We leave a detailed proof in Appendix \ref{app:proof:conv-aud-geo}.
\end{proof}

The convex construction in Corollary \ref{cor:conv-aud-geo} therefore provides a simple and interpretable decomposition: $\rho$ and $\kappa$ describe a query-independent geometry inherent in $\thu$, $M_Q$ transforms it into $(\rho_Q,\lambda_Q^\conv)$, and $U_Q^\conv$ quantifies the first-order approximation error.


\section{Evaluation}\label{sec:eval}

\begin{figure*}[t]
    \centering
    \captionsetup[subfigure]{justification=centering}
    \begin{minipage}[t]{0.49\linewidth}
    \begin{tcolorbox}[groupbox, title=Purchase-100 \cite{Shokri2017MIA}]
        \subfloat[\textbf{GA}, $\eta = 0.5$.\label{sfig:P100-GA-0.5-rand}]
        {\includegraphics[width=.5\linewidth]{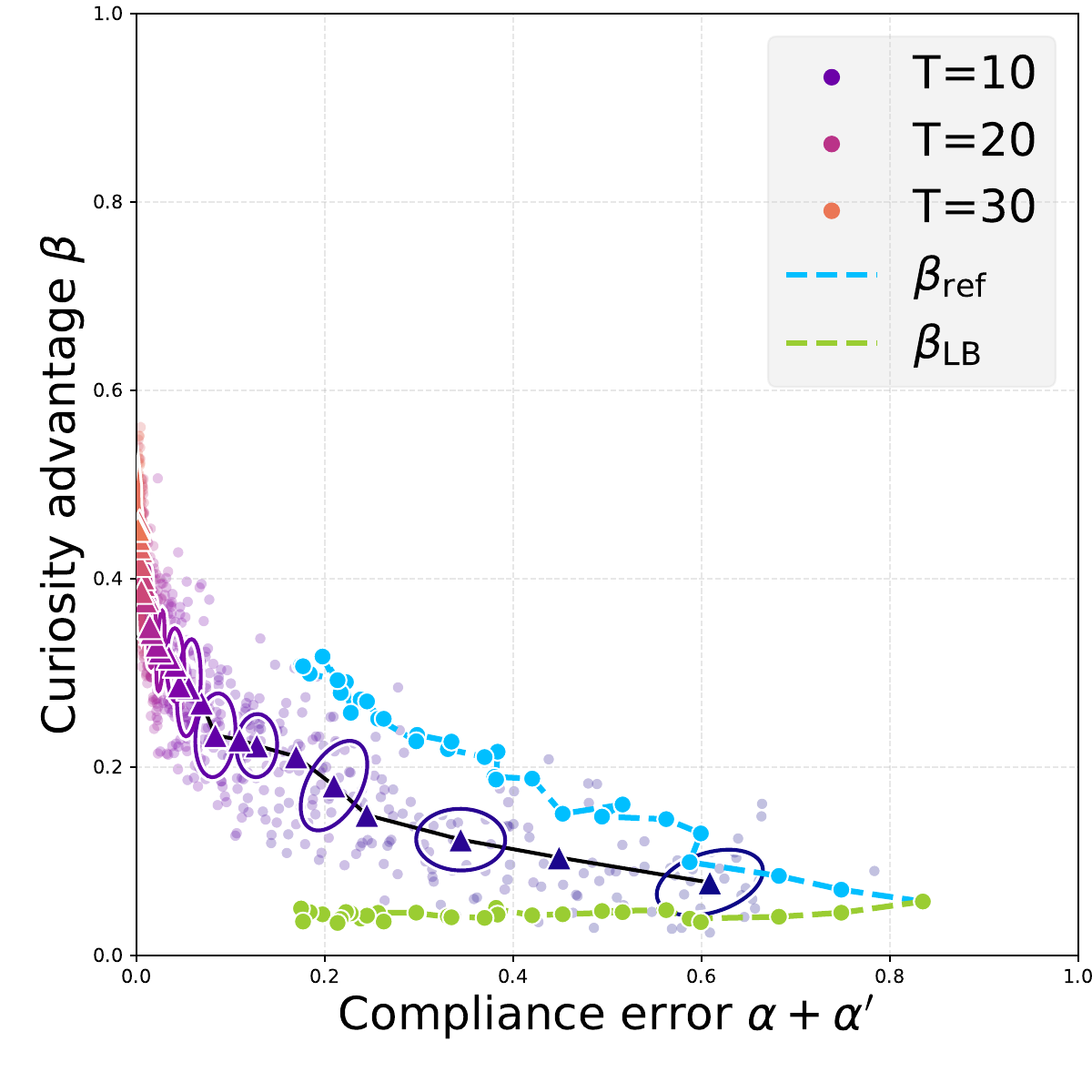}}
        \subfloat[\textbf{SalUn}, $\eta = 0.5$.\\\strut\label{sfig:P100-SalUn-0.5-rand}]
        {\includegraphics[width=.5\linewidth]{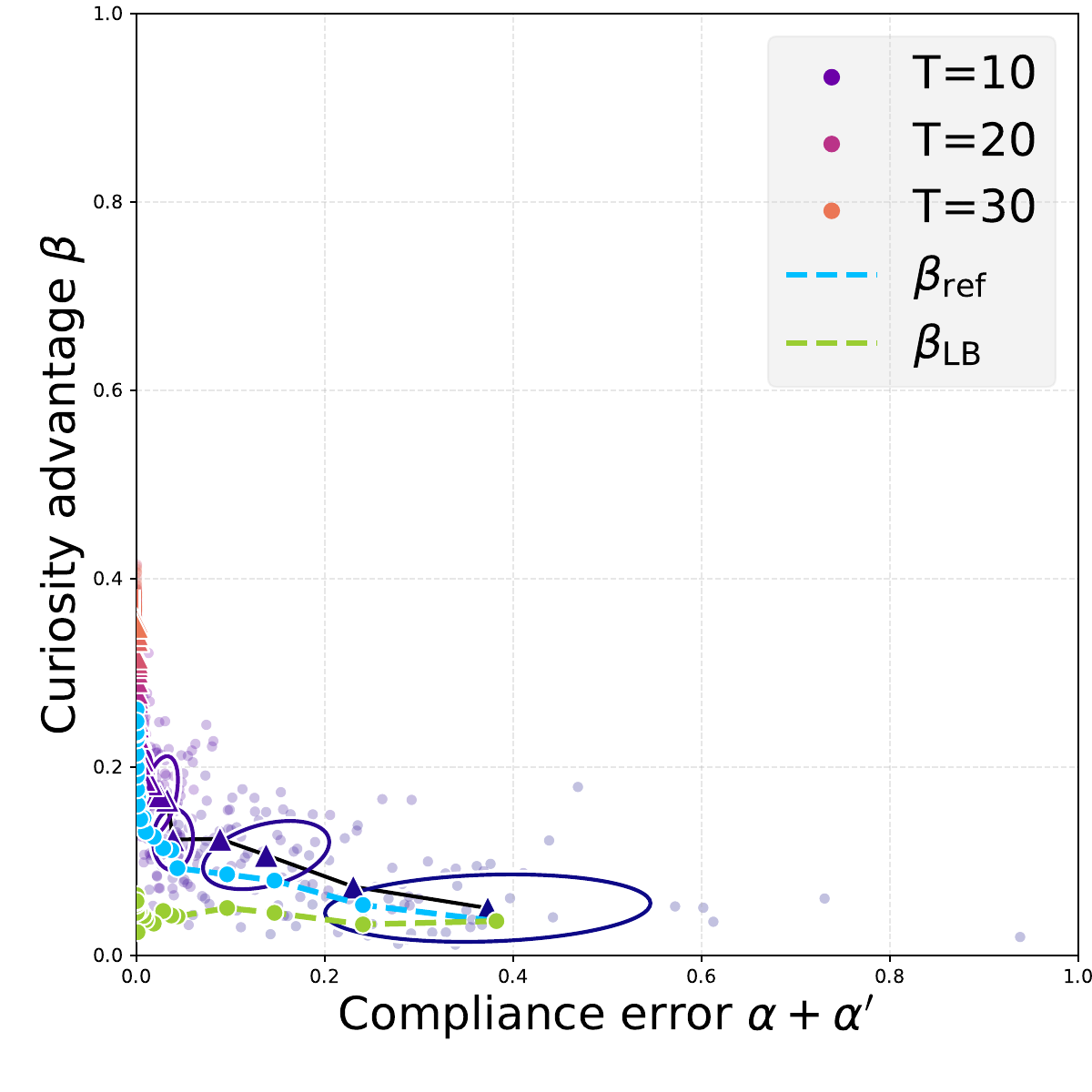}}
        
        \subfloat[\textbf{GA}, $\eta = 0.9$.\\\strut\label{sfig:P100-GA-0.9-rand}]
        {\includegraphics[width=.5\linewidth]{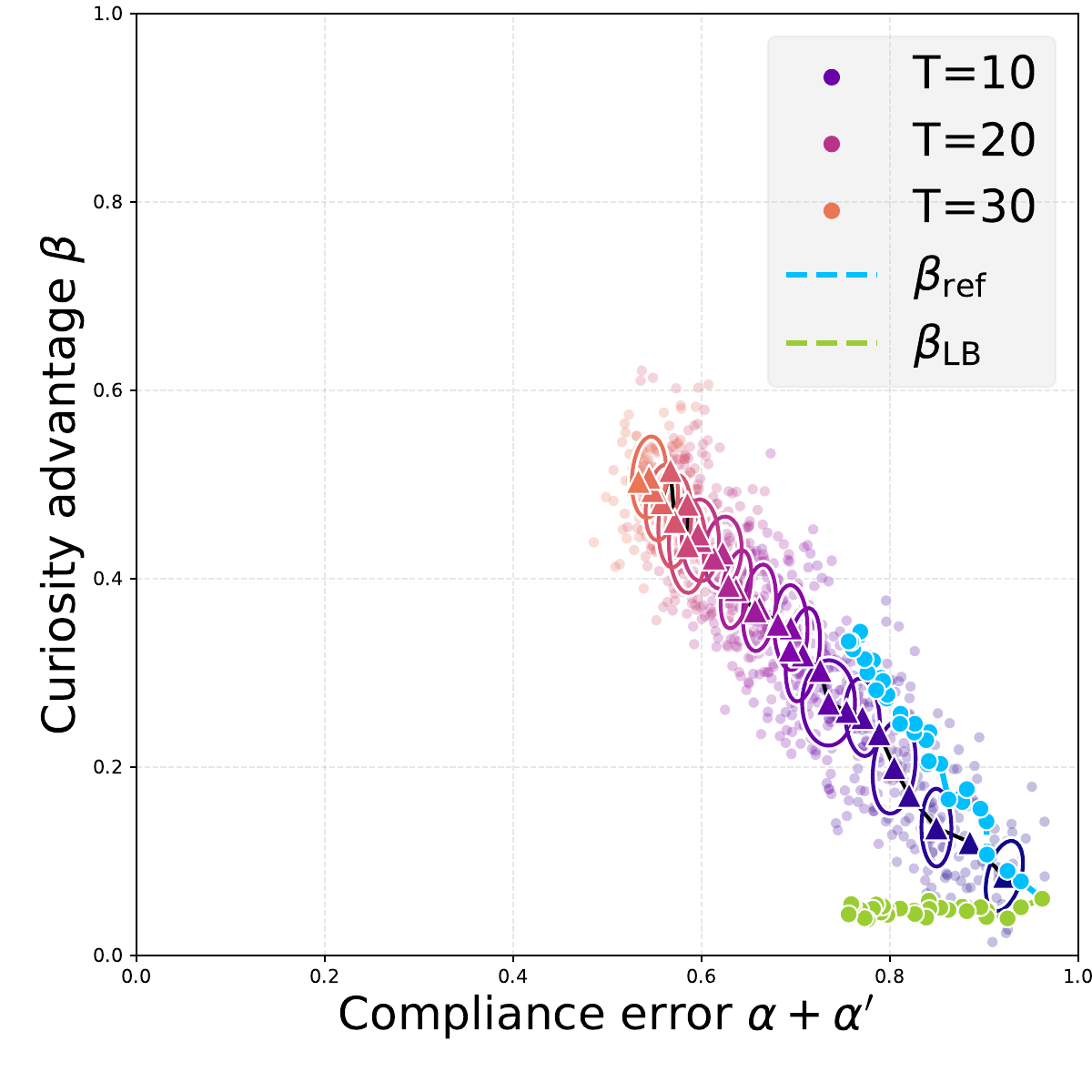}}
        \subfloat[\textbf{SalUn}, $\eta = 0.9$.\\\strut\label{sfig:P100-SalUn-0.9-rand}]
        {\includegraphics[width=.5\linewidth]{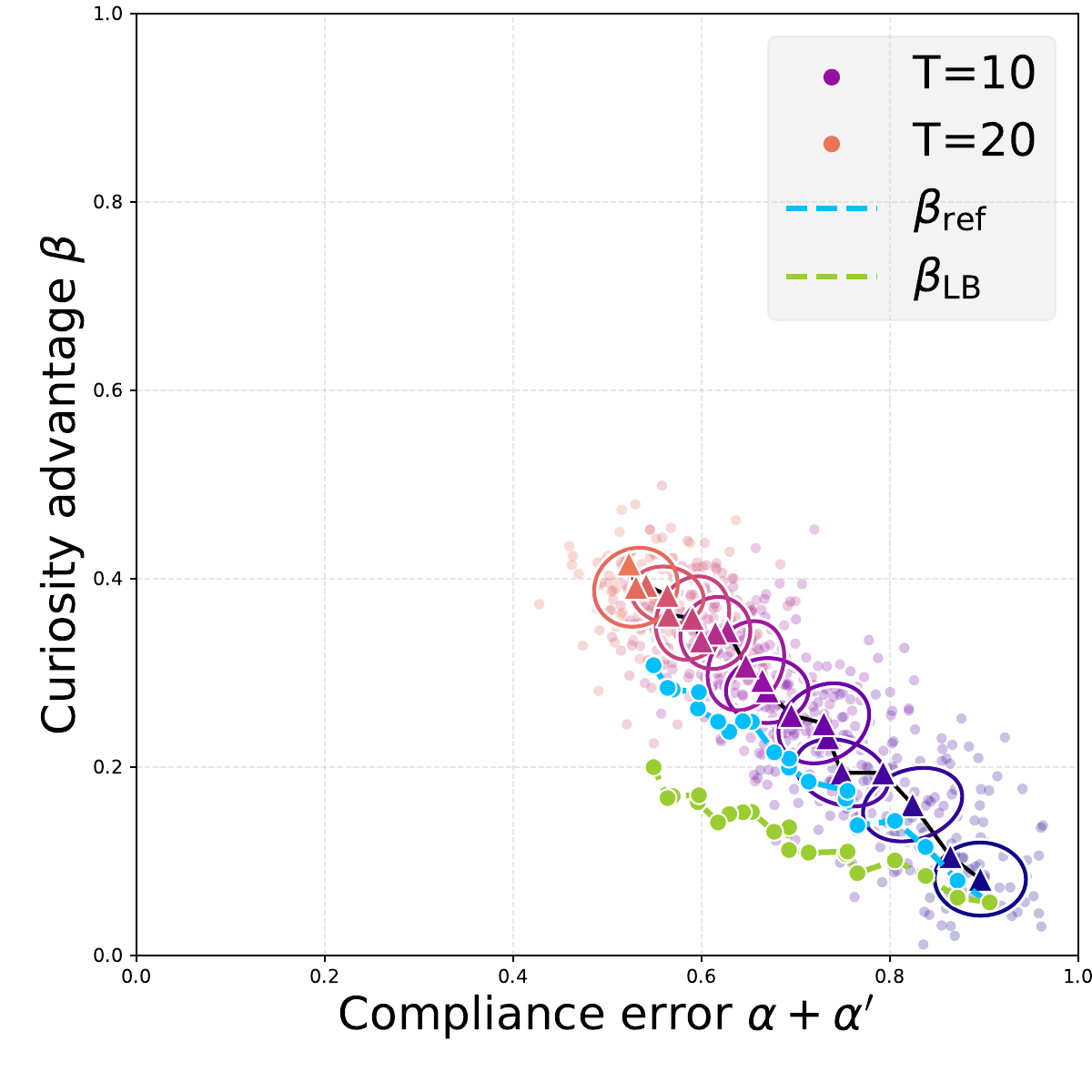}}
    \end{tcolorbox}
    \end{minipage}
    \begin{minipage}[t]{0.49\linewidth}
    \begin{tcolorbox}[groupbox, title=CIFAR-10 \cite{Krizhevsky2009CIFAR}]
        \subfloat[\textbf{GA}, $\eta = 0.5$.\label{sfig:C10-GA-0.5-rand}]
        {\includegraphics[width=.5\linewidth]{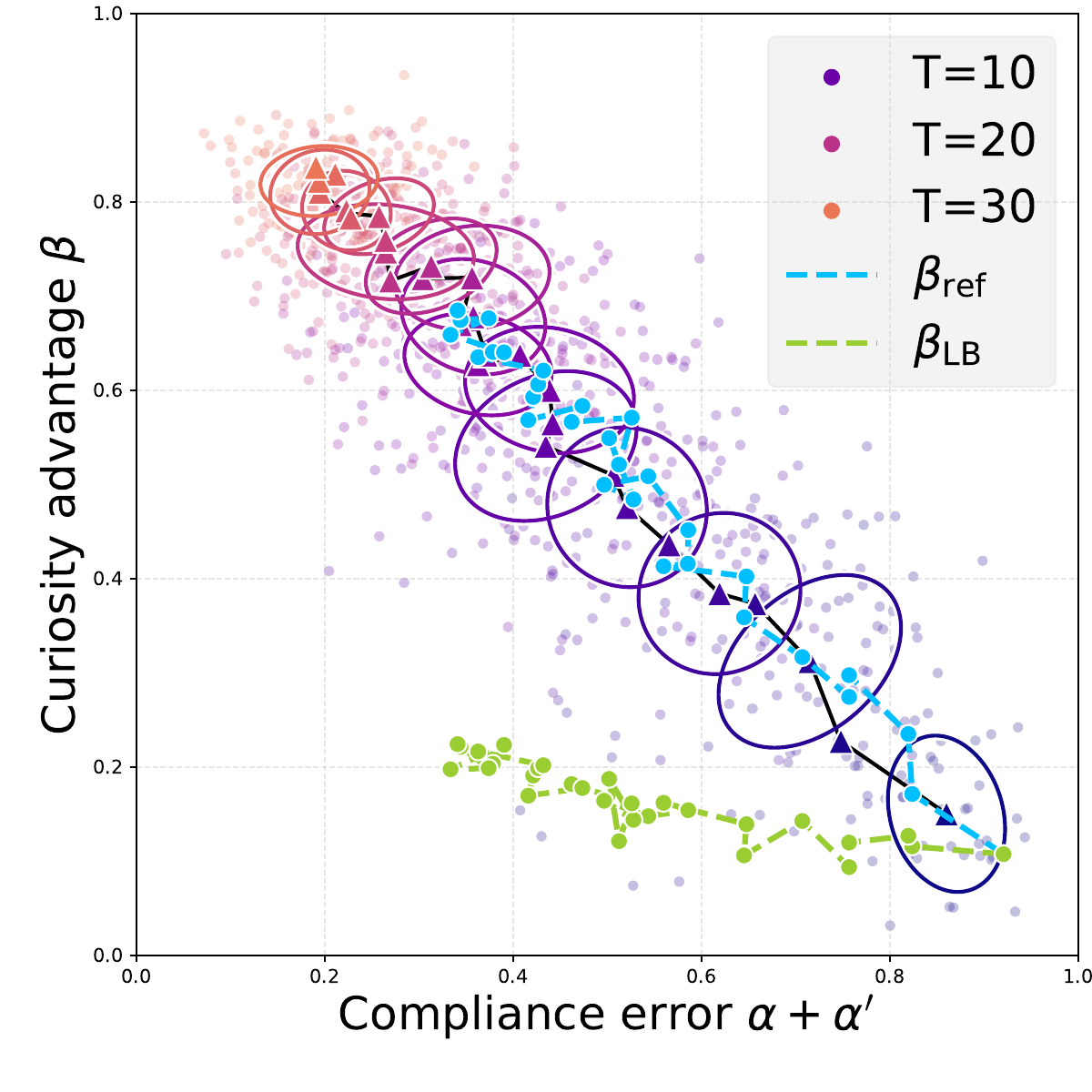}}
        \subfloat[\textbf{GA}, $\eta = 0.5$,\\Class-wise Unlearning\label{sfig:C10-GA-0.5-class}]
        {\includegraphics[width=.5\linewidth]{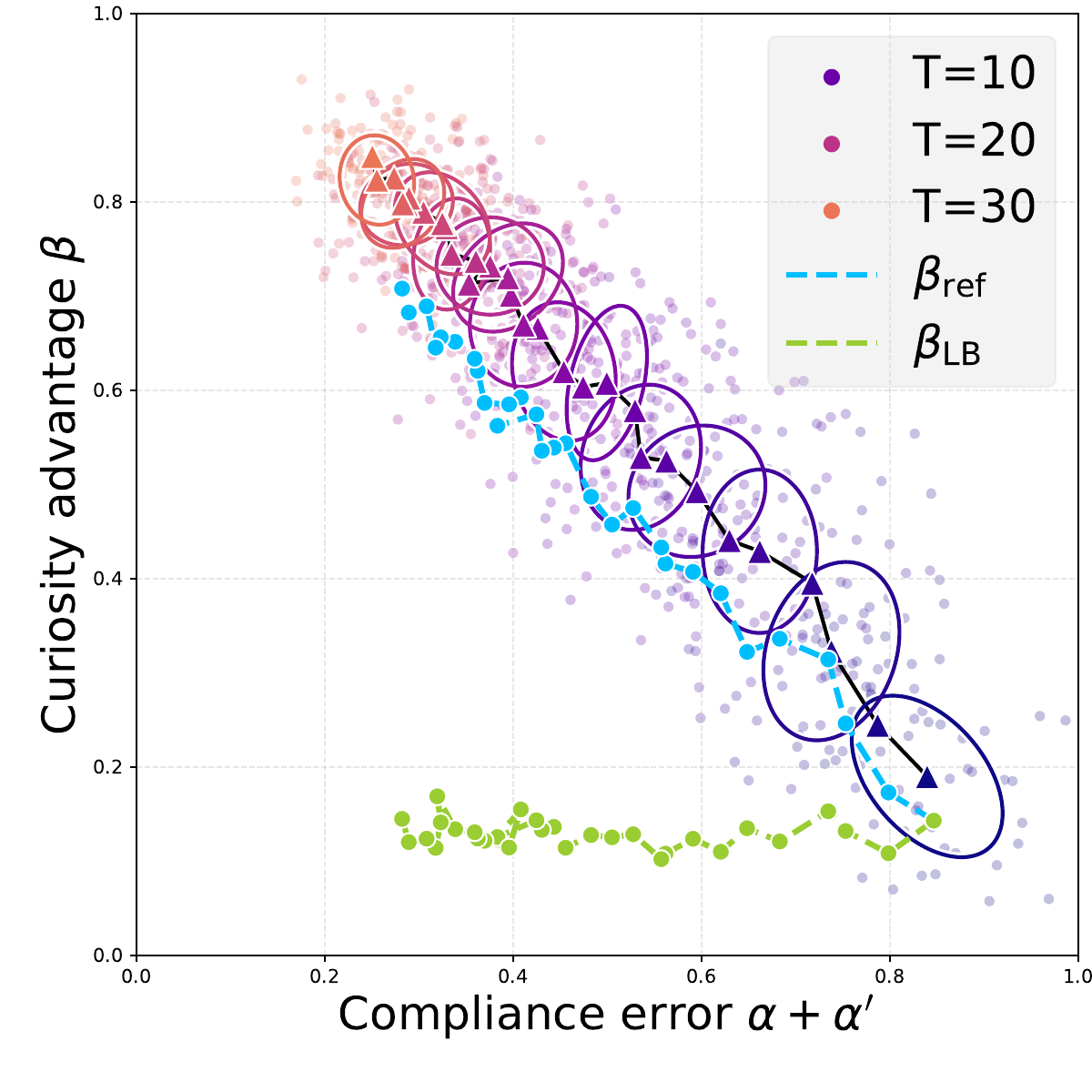}}
        
        \subfloat[\textbf{GA}, $\eta = 0.9$.\label{sfig:C10-GA-0.9-rand}]
        {\includegraphics[width=.5\linewidth]{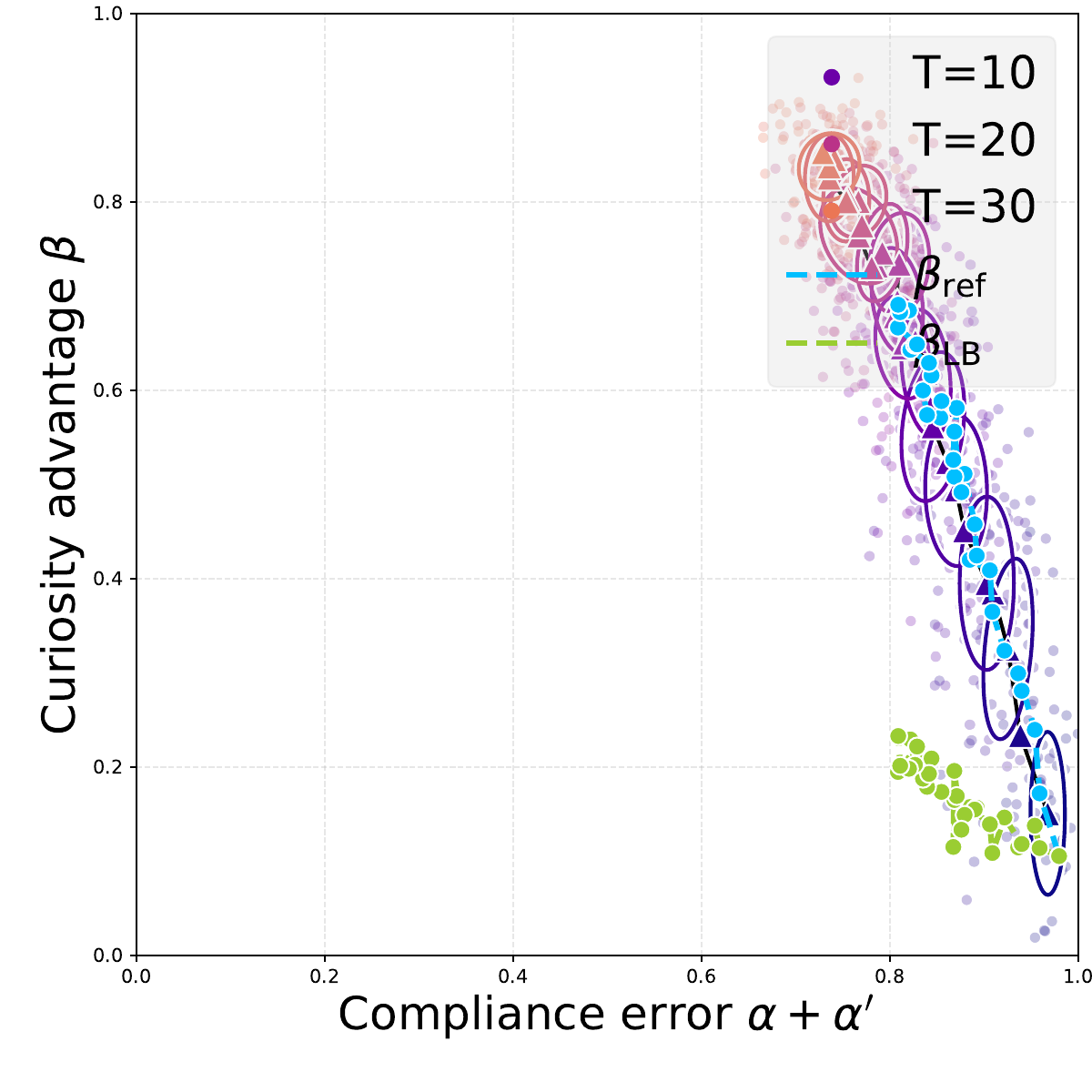}}
        \subfloat[\textbf{GA}, $\eta = 0.9$,\\Class-wise Unlearning\label{sfig:C10-GA-0.9-class}]
        {\includegraphics[width=.5\linewidth]{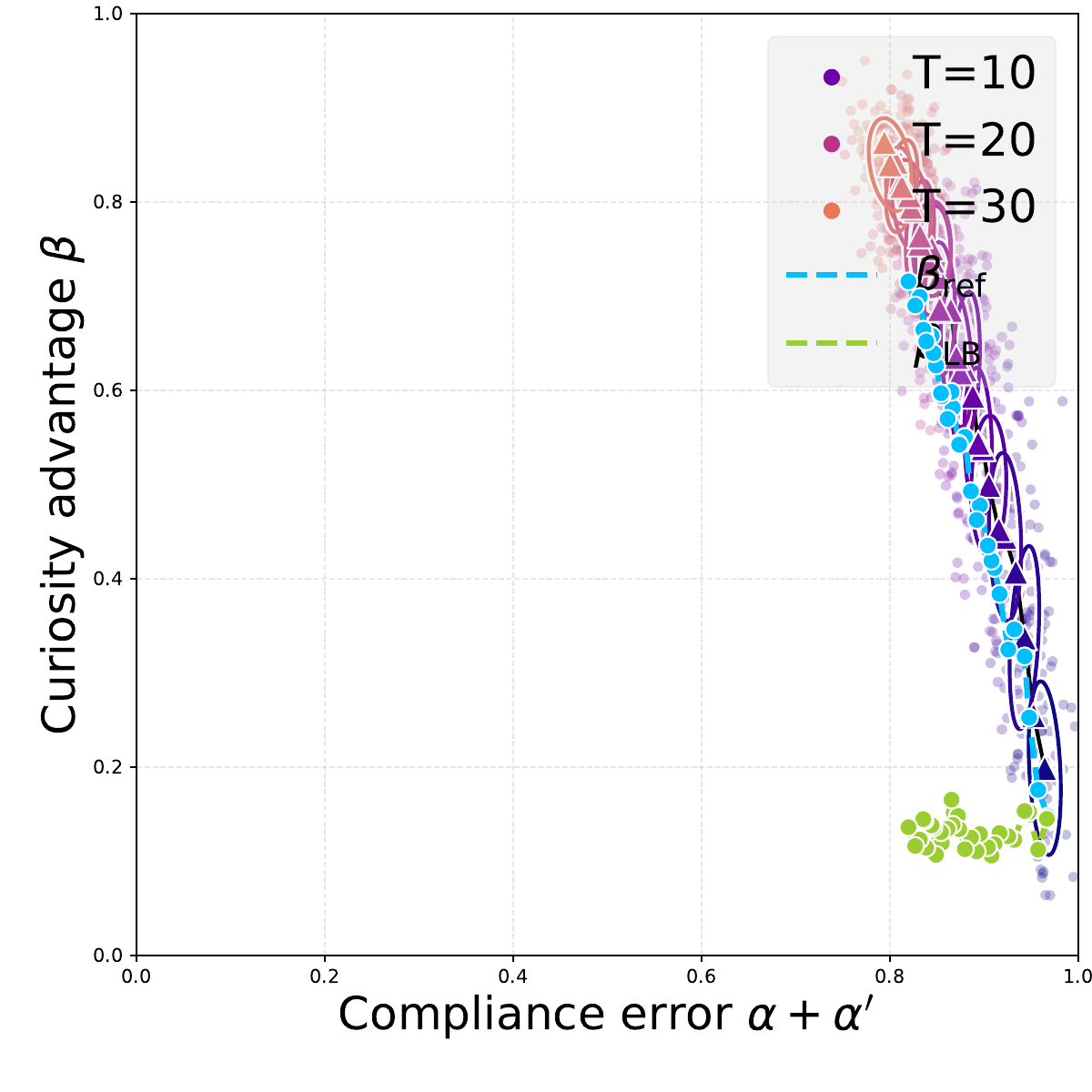}}
    \end{tcolorbox}
    \end{minipage}
    \caption{Pareto curve between $(\alpha + \alpha^\prime, \beta)$ across query budget $1 \leq T \leq 30$ on a fixed $z^*$ for a small 4-layer MLP on Purchase-100 and ResNet-18 on CIFAR-10. We also plot the centroid trajectory $\mu(T)$ and covariance ellipses of one $\sigma_T$ for each $T$. We perform our audit against different models/datasets (Fig. \ref{sfig:P100-GA-0.5-rand}-\ref{sfig:P100-SalUn-0.9-rand} vs \ref{sfig:C10-GA-0.5-rand}-\ref{sfig:C10-GA-0.9-class}), model owner honesty (Fig. \ref{sfig:P100-GA-0.5-rand}, \ref{sfig:P100-SalUn-0.5-rand}, \ref{sfig:C10-GA-0.5-rand}, \ref{sfig:C10-GA-0.5-class} vs \ref{sfig:P100-GA-0.9-rand}, \ref{sfig:P100-SalUn-0.9-rand}, \ref{sfig:C10-GA-0.9-rand}, \ref{sfig:C10-GA-0.9-class}), MU algorithms (Fig. \ref{sfig:P100-GA-0.5-rand}, \ref{sfig:P100-GA-0.9-rand} vs \ref{sfig:P100-SalUn-0.5-rand}, \ref{sfig:P100-SalUn-0.9-rand}), and random-sampling or class-wise unlearning (Fig. \ref{sfig:C10-GA-0.5-rand}, \ref{sfig:C10-GA-0.9-rand} vs \ref{sfig:C10-GA-0.5-class}, \ref{sfig:C10-GA-0.9-class}).}
    \label{fig:pareto-main}
\end{figure*}

\subsection{Experimental Setup}\label{ssec:setup}

To confirm our theoretical results on privacy leakage of target samples transferred from the behavioral audit, we aim to demonstrate the privacy-audit tradeoff established by Theorem \ref{th:main} through membership inference trials.

\paragraph{Models} We examine the audit protocol $\Pi$ on two settings, and train (\rom{1}) 128 surrogate models of a fully-connected neural network, with 4 hidden layers of 1024 dimension for the Purchase-100 dataset\footnote{A subset of  Kaggle's ``acquire valued shoppers'' challenge dataset, see: \url{https://www.kaggle.com/c/acquire-valued-shoppers-challenge/data} and \url{https://github.com/privacytrustlab/datasets} by Shokri et al. \cite{Shokri2017MIA}.} (197,324 samples, 600-dimension binary input, 100 classes) used by Shokri et al. in \cite{Shokri2017MIA}, and (\rom{2}) 32 ResNet-18 models \cite{He2016ResNet} on subsets of the CIFAR-10 dataset \cite{Krizhevsky2009CIFAR} (50,000 training set samples, $3\times32\times32$ image, 10 classes), respectively.

\paragraph{Dataset} The training set of each surrogate model $\tilde{\theta}_i$ contains two parts: (\rom{1}) a shared set $D_u$ that will be unlearned\footnote{We make the explicit assumption in Sec. \ref{ssec:threat-model} that $\Aud$ has knowledge of $D_u$ and therefore can universally unlearn $D_u$ for all surrogate models. Note the subtle distinction between this auditor and the adversary in the U-MIA \cite{Kurmanji2023Unbounded} and U-LiRA \cite{Hayes2024Inexact} attacks without this knowledge, which require more surrogate models with different selections of the unlearned set.} and (\rom{2}) random splits of half of the remaining training set (excluding $D_u$). We set $|D_u| = 1000$ for Purchase-100 and $|D_u| = 500$ for CIFAR-10. We examine both random-sampled unlearning and class-wise unlearning for CIFAR-10; for the latter, $D_u$ is exclusively sampled from class 0. We train the models to 80\% accuracy on a separate validation set.

\paragraph{Unlearning} We unlearn $D_u$ from each model with two representative approximate MU algorithms: Gradient Ascent (\textbf{GA}) \cite{Neel2021Descent} as a baseline and Saliency Unlearning (\textbf{SalUn}) \cite{Fan2024SalUn} as a state-of-the-art unlearning algorithm.
While not a direct adaptation of the $\eta$-honest unlearning (Def. \ref{def:MU-eta}) designed strictly for controlled experiments on convex models (Sec. \ref{ssec:result-convex}), we similarly use $\eta$ to scale the unlearning process and tune the discrepancy between $\thu^\hon$ and $\thu^\dis$: for both \textbf{GA} and \textbf{SalUn}, we unlearn $\lfloor \eta \cdot K \rfloor$ epochs for a dishonest unlearning setting against $K$ epochs in honest unlearning, a smaller $\kappa$ indicate fewer honesty and more discrepancy between $\thu^\hon$ and $\thu^\dis$. We use this setting to simulate dishonest model owners who aim to \emph{under-unlearn} \cite{Tang2025Apollo} the models and preserve unlearned set information in the updated model.

For a more detailed account of the unlearning algorithms, we refer readers to their respective papers \cite{Neel2021Descent,Fan2024SalUn,Huang2024Unified}. We adopt the hyperparameter settings of \textbf{GA} and \textbf{SalUn} used by Huang et al. in \cite{Huang2024Unified} for honest unlearning. We refer to our \href{https://github.com/LiouTang/Behavioral-Unlearn-Audit}{code implementation} for more details.

\paragraph{Adopting LiRA and U-LiRA for Auditing} As noted in Assumption \ref{assum:audit}, the query sequence $Q = (x_1, \dots, x_T)$ are determined independent of $D_u$. 
To perform the audit, we randomly select $T$ samples from a \emph{disjoint validation set} unused in training any surrogate models to query the surrogate models, and collect the posterior probabilities from each model. We repeat the sampling process for each $T$ 30 times to demonstrate the variation in privacy leakage under different selections of $Q$.

Following the LiRA membership inference scheme of Carlini et al. \cite{Carlini2022MIA} and U-LiRA for MU Hayes et al. \cite{Hayes2024Inexact}, we perform two sets of tests for compliance errors ($\alpha$, $\alpha^\prime$) and the curiosity advantage ($\beta(z^*)$). For the former, we aim to distinguish between two worlds: the model is honestly unlearned ($\thu^\hon$), or the model is dishonestly unlearned ($\thu^\dis$); for the latter, between $z^* \in D_r$ ($\thu^\mem$) or $z^* \notin D_r$ ($\thu^\non$).

\paragraph{Calculating a Theoretical $\beta(z^*)$} While the LiRA/U-LiRA-like setting allows us to calculate $(\alpha, \alpha^\prime, \beta)$ over each $Q$ directly, a purely behavioral audit run on surrogate models trained directly forbids us to characterize the Gaussian observation noise $\mathcal{N}(0, \sigma^2)$ noise (Def. \ref{assum:audit}), as the experiments on MLPs and CNNs measure the operational $\beta(z^*)$ directly to observe operational leakage, for which the training process induce model variations hard to characterize as an explicit noise level. We defer a detailed examination of the Gaussian noise to a controlled experiment on \emph{convex} models in Sec. \ref{ssec:result-convex}, where the global behavior- and parameter-space geometry is available, as further empirical confirmation of Sec. \ref{ssec:mechanism}.

Therefore, we need to calculate a theoretical $\beta(z^*)$ directly from the surrogate models. Recall the compliance and query vectors $c_Q$ and $p_Q$ given in Def. \ref{def:sig-vectors}; we separate all (unlearned) surrogate models into four categories:
\begin{equation}
    \thu^{\hon, \mem}, \quad \thu^{\hon, \non}, \quad \thu^{\dis, \mem}, \quad \thu^{\dis, \non}.
\end{equation}
For each category of models, write e.g.: 
\begin{equation}
    \mu^{\hon, \mem} = \mathbb{E} [ F_Q (\vartheta_u) \mid \hon, \mem ],
\end{equation}
and the rest follows. Then we can write:
\begin{equation}
\begin{aligned}
    \tilde{c}_Q &= \frac{1}{2}\left[
    (\mu^{\hon, \mem}-\mu^{\dis, \mem}) + (\mu^{\hon, \non}-\mu^{\dis, \non})
    \right],\\
    \tilde{p}_Q &= \frac{1}{2}\left[
    (\mu^{\hon, \mem}-\mu^{\hon, \non}) + (\mu^{\dis, \mem}-\mu^{\dis, \non})
    \right],\\
\end{aligned}
\end{equation}
and subsequently,
\begin{equation}\label{eq:gamma-operational}
    \tilde{\gamma}_Q = \frac{|\langle \tilde{p}_Q,\tilde{c}_Q\rangle|}{\|\tilde{c}_Q\|^2}, \quad
    p_Q^\perp = \tilde{p}_Q - \frac{\langle \tilde{p}_Q,\tilde{c}_Q\rangle \tilde{c}_Q}{\|\tilde{c}_Q\|^2},
\end{equation}
are derived independently of $\sigma$. Write $\widehat{\Sigma} = \mathrm{diag}(\hat{v}_1,\dots,\hat{v}_n)$ as the common diagonal residual covariance of vectors $v_1, \dots, v_n$. Then we can define $d_\Comp$ and $d_\Cur$ as their Mahalanobis distances, i.e.:
\begin{equation}
    \tilde{d}_\Comp = \left\| \widehat{\Sigma}^{-1/2}\tilde{c}_Q \right\|, \quad
    \tilde{d}_\Cur  = \left\| \widehat{\Sigma}^{-1/2}\tilde{p}_Q \right\|,
\end{equation}
from which we can estimate a referential $\beta_\mathrm{ref}$ as:
\begin{equation}\label{eq:beta-ref}
\begin{aligned}
    e_\mathrm{ref} &= 2\left[ 1- \Phi ({\tilde{d}_\Comp}/2) \right], \\
    \beta_\mathrm{ref}(z^*) &= 2 \Phi (\tilde{d}_\Cur / 2) - 1.
\end{aligned}
\end{equation}



\subsection{Results}\label{ssec:result-non-convex}

\paragraph{Privacy-Audit Tradeoff} Following the experimental setup described in Sec. \ref{ssec:setup}, we demonstrate the results in Fig. \ref{fig:pareto-main}. For each Pareto curve under one model architecture, $D_u$, $\unA$, we also include:
\begin{itemize}
    \item The theoretical referential $\beta_\mathrm{ref}(z^*)$ at each $T$ similar to Eq. \ref{eq:alpha-beta-d} by plotting the \emph{centroid} of $(e_\mathrm{ref}, \beta_\mathrm{ref})$ over the 30 sampled query sequences for each value of $T$;
    \item The theoretical lower bound $\beta_{\mathrm{LB}}(T) \leq \beta(z^* \mid T)$, calculated as $\beta_{\mathrm{LB}} = 2\Phi (\tilde{\gamma}_Q q_e) - 1$ following Eq. \ref{eq:bata-main-lb}, $\gamma_Q$ is given by our construction in Eq. \ref{eq:gamma-operational}, $q_e=\Phi^{-1}(1-e/2)$ is calculated from the Monte Carlo membership inference trials and similarly taken as the {centroid}.
\end{itemize}
\noindent We also include additional results in Appendix \ref{app:add-exp}.

Our results in Fig. \ref{fig:pareto-main} demonstrates the predicted behavioral coupling of Theorem \ref{th:main}, in which an $(\alpha+\alpha^\prime, \beta)$-tradeoff is present across all experiments: as the behavioral audit achieves lower error rate, privacy leakage on $z^* \in D_r$ observable to the honest-but-curious auditor $\Aud$ is more significant. The empirical results derived from direct membership inference trials closely follow the theoretical $\beta_\mathrm{ref}$, demonstrating that the geometrical transfer successfully accounts for privacy leakage on $z^*$.
Several trends can additionally be observed across different datasets, unlearning algorithms, and model owner honesty levels:
\begin{itemize}
    \item Under a fixed query budget $T$, the achievable $(\alpha+\alpha^\prime, \beta)$ lie within a range, confirming that there exist geometrical bounds that constrain the ability of $\Aud$. By increasing the query budget, $\Aud$ rapidly reduces the compliance error and simultaneously gains curiosity advantage, where $\eta=0.1$ (see Appendix \ref{app:add-exp}) can be trivially distinguished by only 10 randomly selected queries.
    \item The variance induced by different queries also decreases with larger $T$, as the marginal information gained by individual observations diminishes.
    \item The tradeoff is additionally influenced by a variety of factors: \textbf{SalUn} incurs more privacy leakage at the same compliance error on the model/dataset selections (Fig. \ref{sfig:P100-GA-0.5-rand}, \ref{sfig:P100-GA-0.9-rand} vs \ref{sfig:P100-SalUn-0.5-rand}, \ref{sfig:P100-SalUn-0.9-rand}); further, class-wise unlearning is more susceptible to $\Aud$, as $D_u$ is more homogeneous, $\gamma_Q(z^*)$ naturally increases for $z^*$ of the same class (class 0).
\end{itemize}

Most interestingly, $\beta_\mathrm{LB}$ stays ineffective across most of the experiments except for one in Fig. \ref{sfig:P100-SalUn-0.9-rand}.
To explain this effect, we calculate $R_\perp = {\|\tilde{p}_Q^\perp\|^2} / {\|\tilde{p}_Q\|^2}$, which gives an average $ \overline{R_\perp} \approx 0.8$ at $T = 10$ across 30 selections of $Q$ for Fig. \ref{sfig:P100-GA-0.5-rand}, i.e., the curiosity signal are near-orthogonal to the compliance signal $c_Q$ even with small-sized query sequences. As such, $\beta(z^*) \approx \beta_\mathrm{ref} \gg \beta_\mathrm{LB}$: while the referential leakage derived from the endpoint geometry is aligned with the empirical $\beta$, the theoretical lower-bound given by Theorem \ref{th:main} is intentionally conservative and only provides the privacy leakage component \emph{forced} by $Q$. We study in Sec. \ref{ssec:result-convex} how additional structural assumptions on the model parameter space, given by convexity, allow us to analyze this latent geometry.


\begin{figure*}[t]
    \centering
    \captionsetup[subfigure]{justification=centering}
    \subfloat[$f_{\thu}(x)$ over $\mathbb{R}^2$\label{sfig:dec-space}]
    {\includegraphics[width=.25\linewidth]{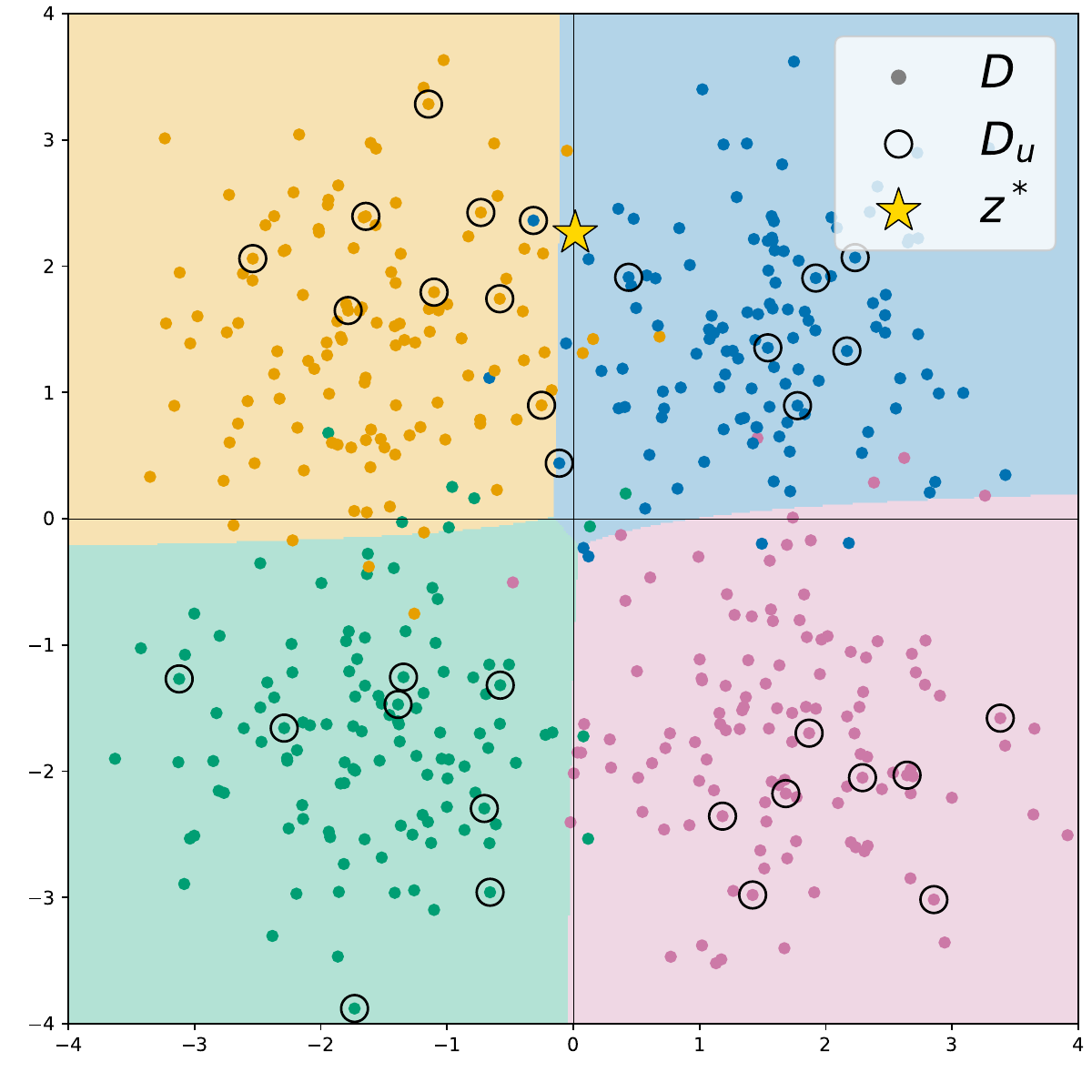}}
    \subfloat[$\eta = 0.1$\label{sfig:eta-0.1}]
    {\includegraphics[width=.25\linewidth]{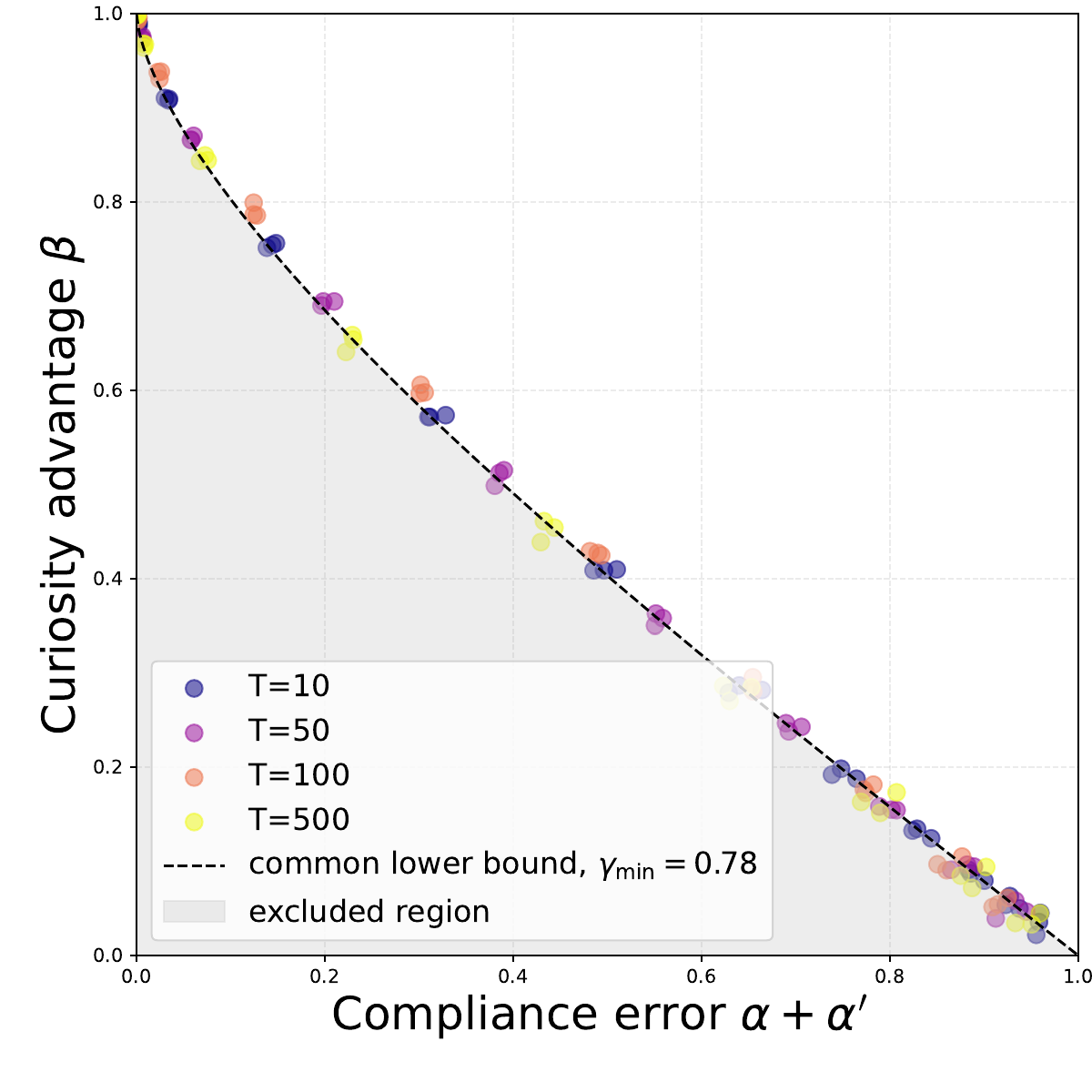}}
    \subfloat[$\eta = 0.5$\label{sfig:eta-0.5}]
    {\includegraphics[width=.25\linewidth]{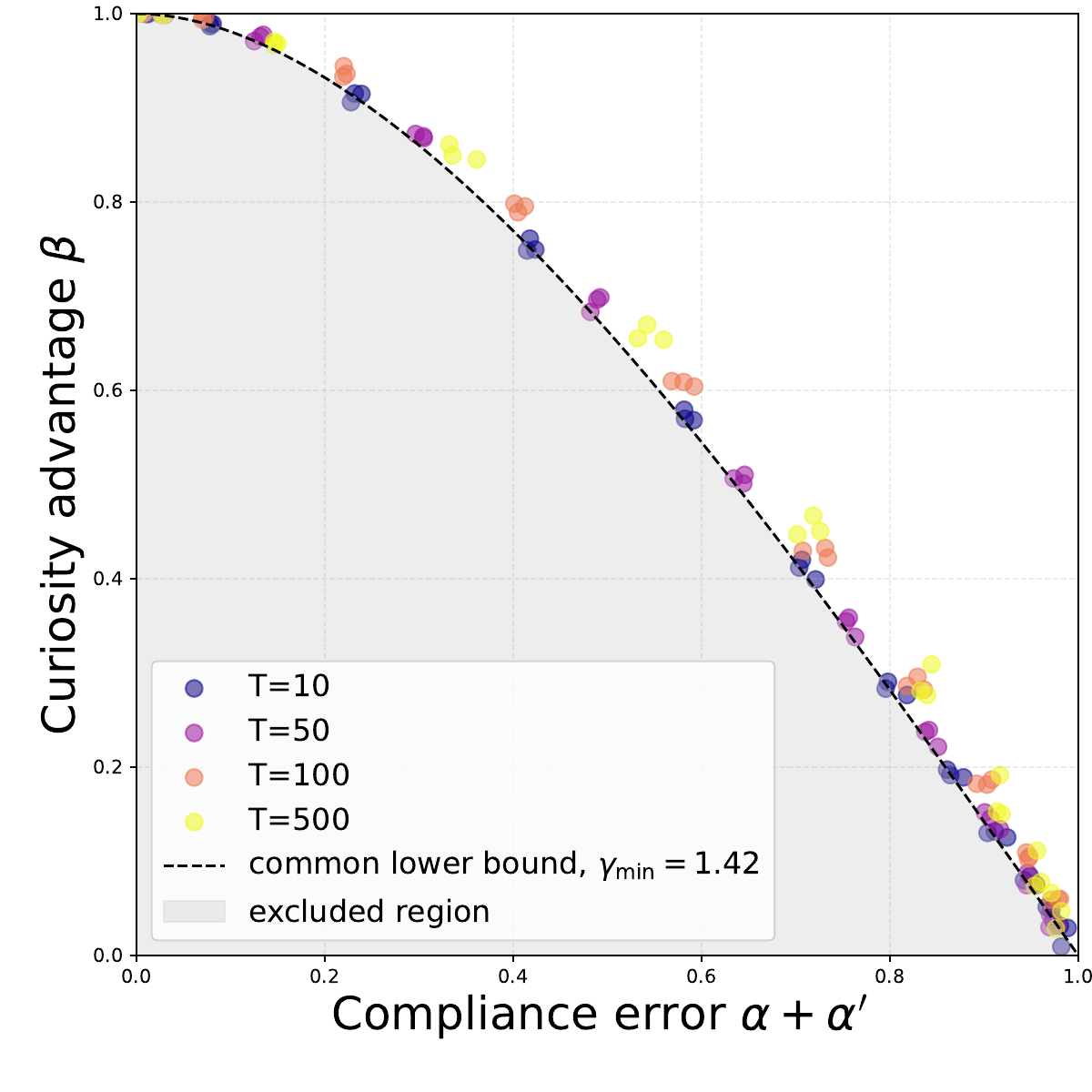}}
    \subfloat[$\eta = 0.9$\label{sfig:eta-0.9}]
    {\includegraphics[width=.25\linewidth]{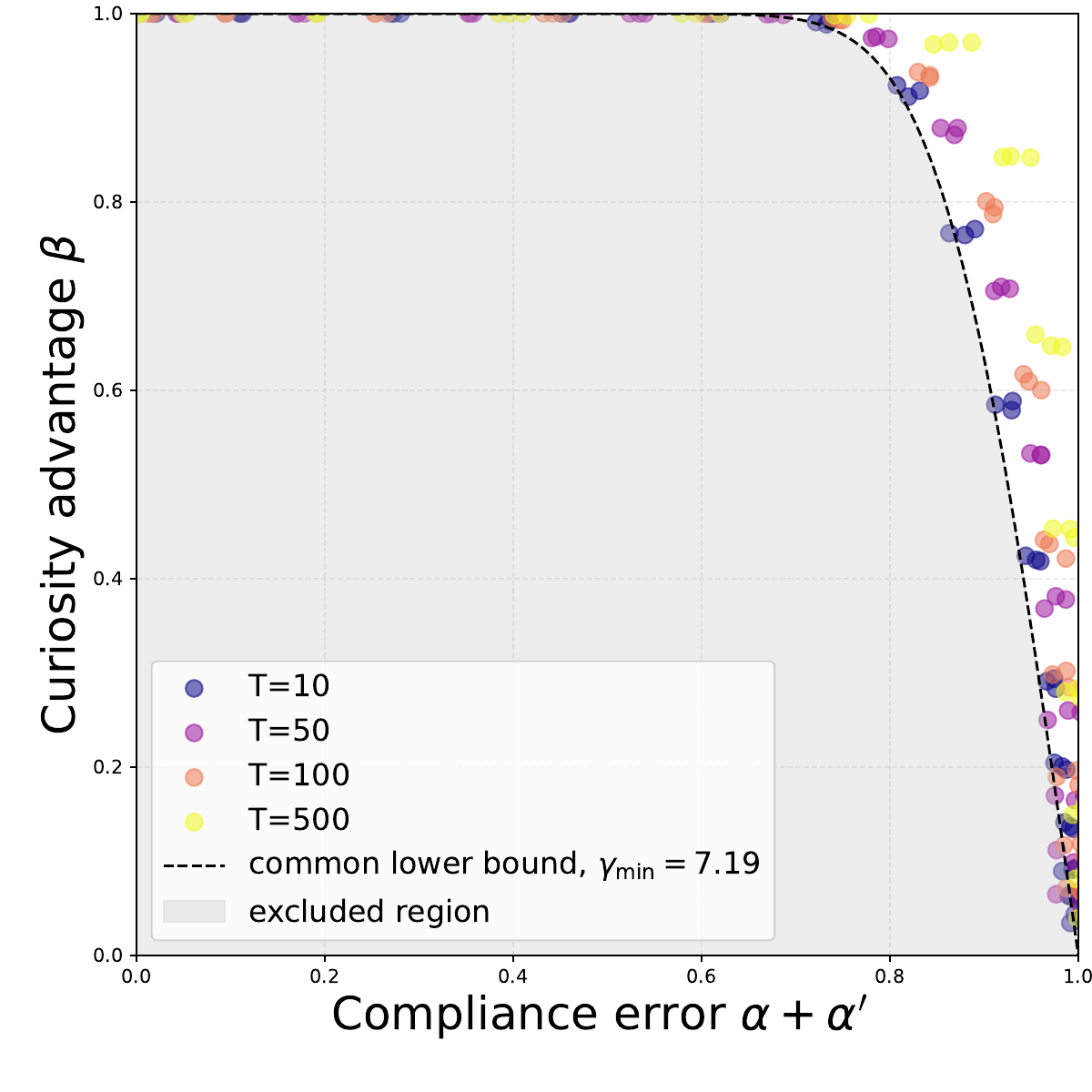}}
    \caption{The target LR model (Fig. \ref{sfig:dec-space}) and its Pareto curve between $(\alpha + \alpha^\prime, \beta)$ across observational noise level $0.005 \leq \sigma \leq 0.5$ and query budget $T = 10,50,100,500$ on a fixed $z^*$ (Fig. \ref{sfig:eta-0.1}-\ref{sfig:eta-0.9}). We also show the theoretical \emph{closed-form} $\beta$ (Eq. \ref{eq:beta-main-exact}).}
    \label{fig:pareto-convex}
\end{figure*}

\begin{figure*}[t]
    \centering
    \captionsetup[subfigure]{justification=centering}
    \subfloat[$\gamma_Q(\eta)$ and $\lambda_Q$ over $\eta \in [0,1)$\\at a fixed $z^*$ and $D_u$.\label{sfig:coeff-eta}]
    {\includegraphics[width=.25\linewidth]{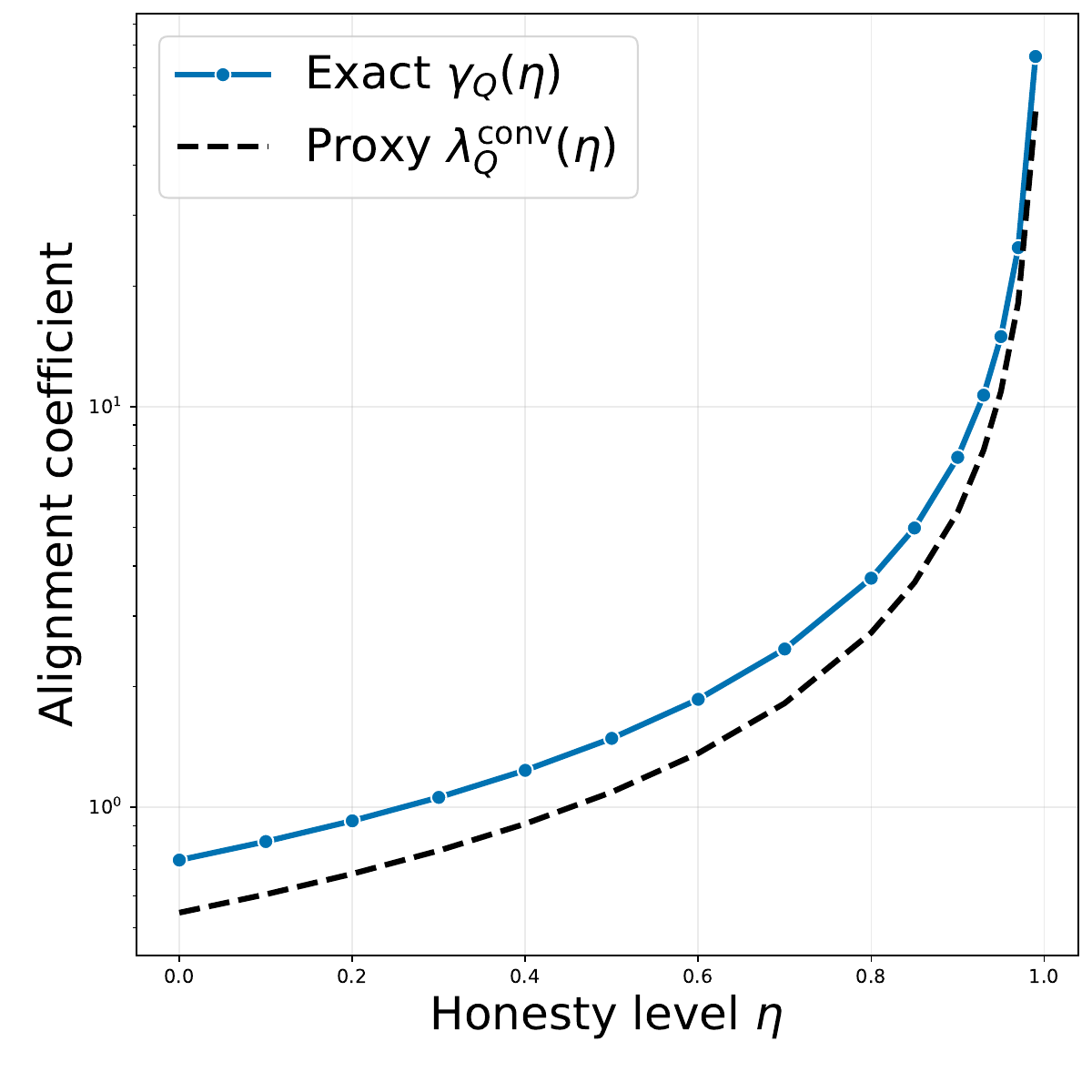}}
    \subfloat[$\gamma_Q(z^*)$ over $z^* \in D_r$\\at a fixed $D_u$, $\eta=0$.\label{sfig:gamma-distribution}]
    {\includegraphics[width=.25\linewidth]{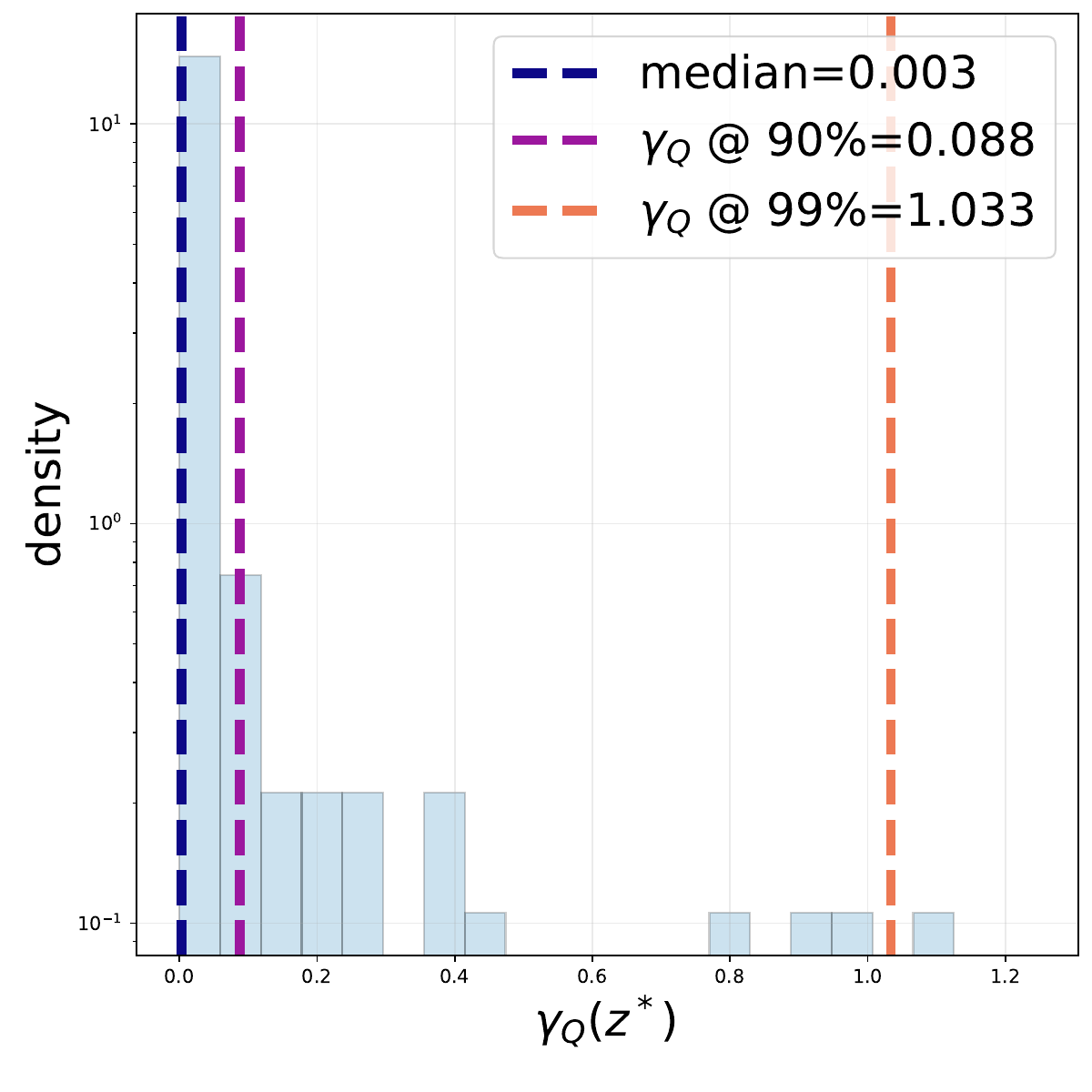}}
    \subfloat[$\gamma_Q$ over different selections of $Q$\\at a fixed $z^*$, $D_u$, $\eta=0$.\label{sfig:gamma-Q-sel}]
    {\includegraphics[width=.25\linewidth]{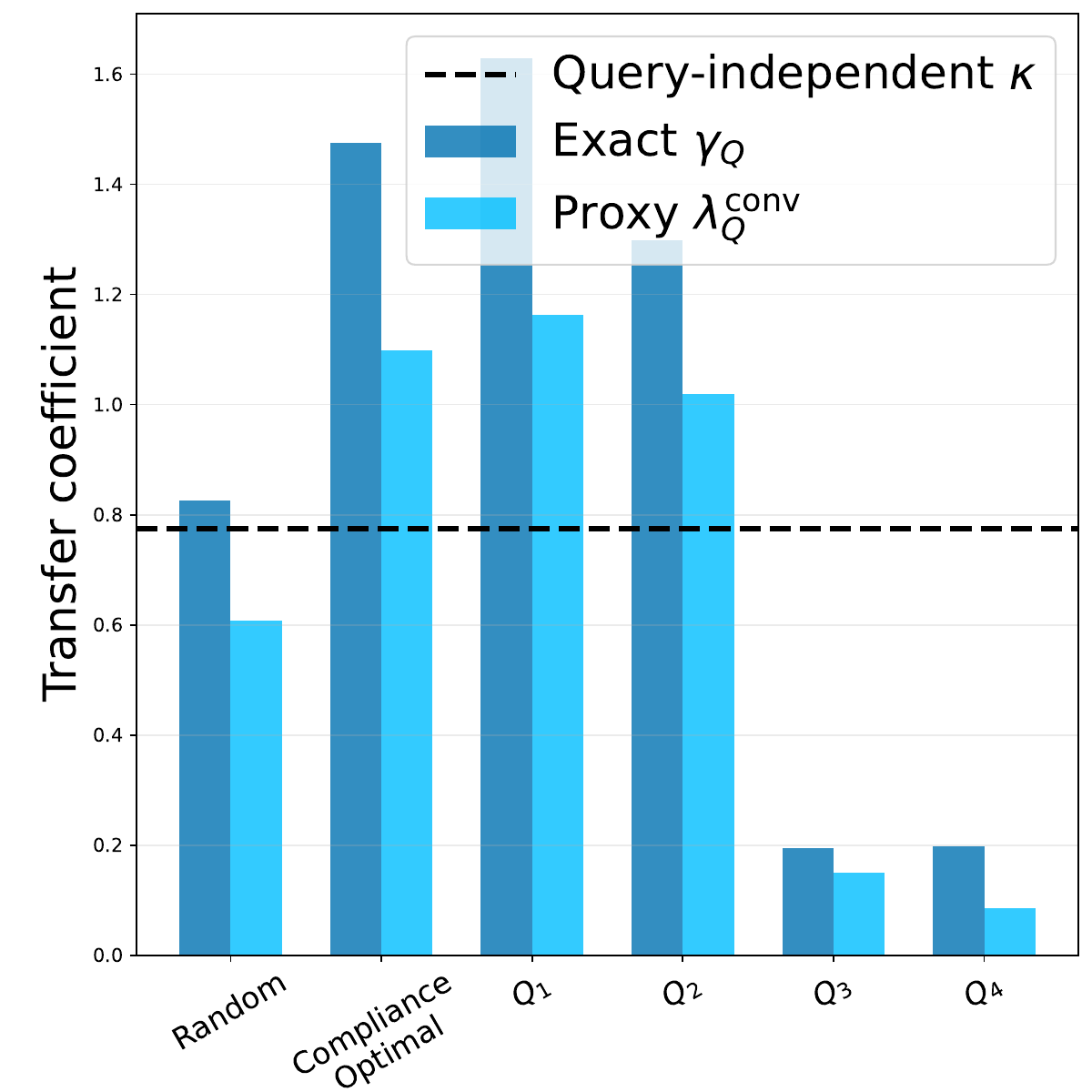}}
    \subfloat[$\gamma_Q(D_u)$ over $D_u \subseteq D$\\at a fixed $z^*$, $\eta=0$.\label{sfig:gamma-D_u}]
    {\includegraphics[width=.25\linewidth]{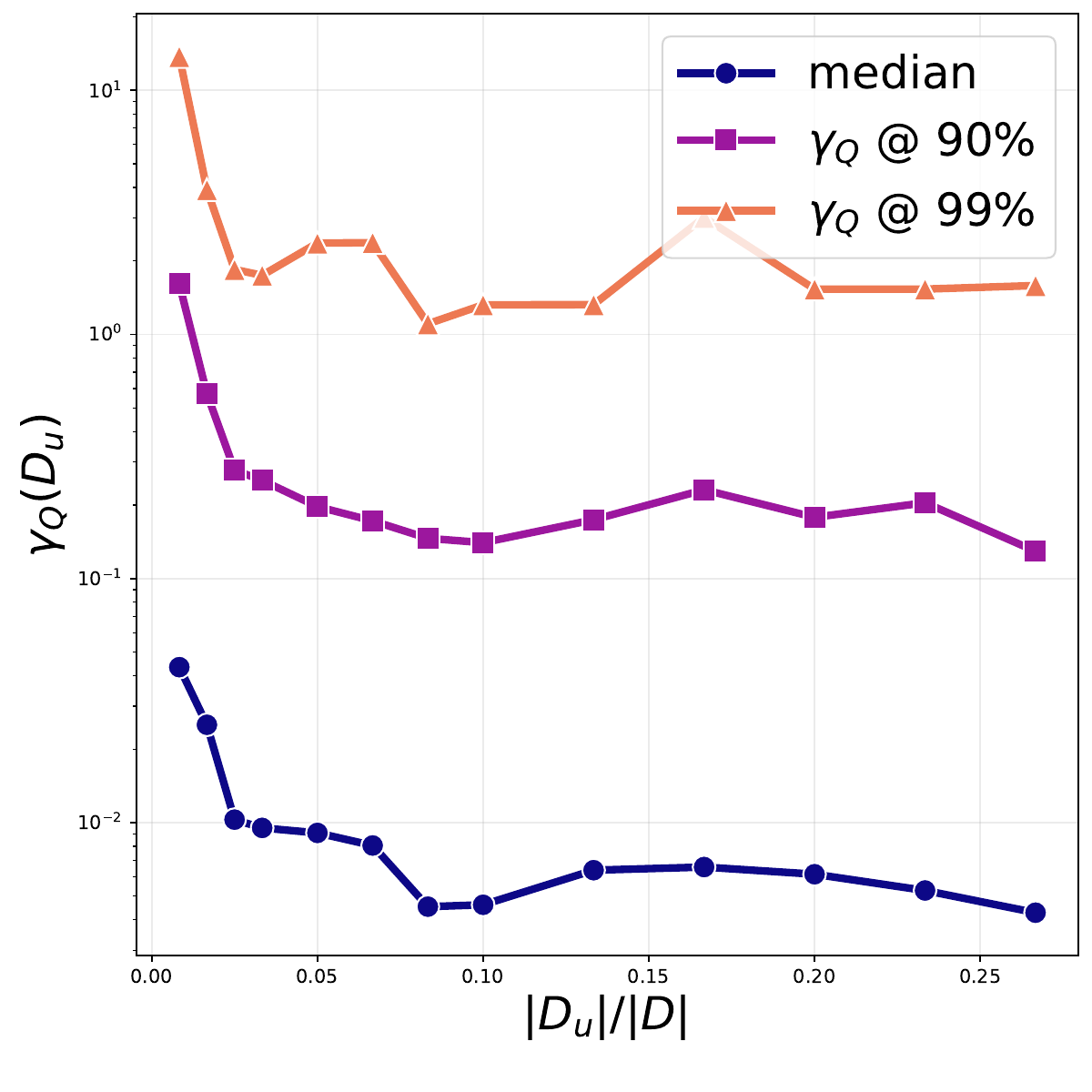}}
    \caption{The trend of $\gamma_Q$ at different settings of $\eta$, $z^*$, $Q$, and $D_u$.}
    \label{fig:gamma-Q}
\end{figure*}

\subsection{Results on Convex Models}\label{ssec:result-convex}

As discussed in Sec. \ref{ssec:result-non-convex}, while directly calculating $(\alpha+\alpha^\prime, \beta)$ on \emph{arbitrary} MLP/CNN models in the wild allows us to confirm the geometry in Theorem \ref{th:main}, a lack of the global parameter space prevents us from calculating the \emph{closed-form} theoretical $\beta$ given in Eq. \ref{eq:beta-main-exact}. Therefore, we here provide further empirical confirmations of our theoretical results (Sec. \ref{ssec:bridge} through \ref{ssec:mechanism}) with controlled experiments on a small-scale \emph{convex} model following Assumption \ref{assum:convex} and Def. \ref{def:MU-eta}. We aim not to scale the effects of results in Sec. \ref{ssec:result-non-convex}, but to have a more interpretable environment that allows us to observe how $z^*$, $D_u$, $Q$, and other elements of the behavioral audit (Assumption \ref{assum:audit}) affect its outcome.

\paragraph{Dataset} We train all convex models on a two-dimensional dataset $(x, y) \in \mathbb{R}^2 \times \{0, 1, 2, 3\}$. We generate \emph{class-balanced} samples across $4$ classes, each class following a 2-dimensional Gaussian distribution that roughly aligns with different quadrants, as shown in Fig. \ref{sfig:dec-space}.

\paragraph{Models} We train a multinomial \emph{logistic regression} model as the classifier with damped Newton to convergence. We use $\ell_2$ regularization with $\lambda = 5\times 10^{-3}$ (Eq. \ref{eq:convex-fit}) to ensure strong convexity on the model. For feature lift, we use $\phi(x) = (x_1, x_2, x_1 x_2, |x_1|, |x_2|, 1)$, which is deliberately \emph{non-convex in $x$} consistent with the settings used by existing works \cite{Guo2020Cert,Neel2021Descent}.

\paragraph{Unlearning} Following Def. \ref{def:MU} and Assumption \ref{assum:convex}, We unlearn the target model(s) through the $(\varepsilon, \delta)$-certified unlearning algorithm proposed by Guo et al. \cite{Guo2020Cert}, which ensures that a baseline honest unlearning $\unA^\hon$ is sufficient. We sample a class-balanced unlearned set $D_u$ uniformly from each class.


\paragraph{Privacy-Audit Tradeoff} As in Sec. \ref{ssec:setup}, we perform Monte Carlo trials on the convex models; Fig. \ref{fig:pareto-convex} shows the results.

Our results confirm the theoretical $\beta(z^*)$ given in Eq. \ref{eq:beta-main-exact} \emph{nearly exactly} across different settings of $(\eta, \sigma, T)$: $\Aud$ has a curiosity advantage on the membership $z^* \in D_r$ that scales with the privacy-audit alignment coefficient characterized by $\gamma_Q$. For the same pair of $(T, \sigma)$, $\alpha+\alpha^\prime$ increases with a higher $\eta$ as the discrepancy between $\thu^\hon$ and $\thu^\dis$ becomes harder to detect while yielding no additional membership advantage on $\beta(z^*)$, pushing the Pareto curve to the right. Noteworthily, even when $\Aud$ has no prior knowledge of the distributions $P_\tau \mid{\{ \langle \hon, \dis \rangle \times \langle \mem, \non \rangle \}}$, the fitted Gaussian distributions allow the auditor to achieve near-optimal separation between different signals \cite{Carlini2022MIA,Hayes2024Inexact}.

More interestingly, we note that the Pareto curve for different combinations of $(T, \sigma)$ collapses. Recall that, $M_Q$ (Eq. \ref{eq:MQ}) determines how query-independent geometry $(\rho, \kappa)$ is transformed into $(\rho_Q,\lambda_Q^\conv)$, for i.i.d queries, we have:
\begin{equation}
    M_Q/T = \frac{1}{T}\sum_{t=1}^T h_{x_t}h_{x_t}^\top \to \mathbb{E}[h_xh_x^\top],
\end{equation}
that is, both compliance and curiosity \emph{norms} grow proportional to $\sqrt{T}$ while $\rho_Q$ and the ratio defining $\lambda_Q^\conv$ stabilize. Further, $M_Q$ is influenced by a coupled factor $\sqrt T/\sigma$, causing the curves under different $(T, \sigma)$ to collapse under the same $\eta$ (Fig. \ref{sfig:eta-0.1}-\ref{sfig:eta-0.9}). A higher query budget $T$ for $\Aud$ is operationally equivalent to reducing observational noise $\sigma$.

\paragraph{Transfer Coefficient $\bm{\gamma_Q}$} To gain a more intuitive understanding of the relationship between the compliance and curiosity signals, we further examine the privacy-audit transfer coefficient $\gamma_Q$ and how it's affected by different parts of the audit process: model owner's honesty level $\eta$, target $z^*$, query sequence $Q$, and the unlearned ser $D_u$. We plot $\gamma_Q$ and its proxy through the convex geometry $\lambda_Q^\conv$ in Fig. \ref{fig:gamma-Q}.

Our results confirm our theoretical intuition in Corollary \ref{cor:conv-aud-geo}: Fig. \ref{sfig:coeff-eta} traces $\gamma_Q$ and $\lambda_Q^\conv$ over $\eta$, both scales as $1/(1 - \eta)$ for effective compliance queries following Eq. \ref{eq:lambda-conv}.
Noteworthily, $\gamma_Q$ diverges with $\eta \to 1$, as the compliance displacement $\delta_c = \thu^\hon - \thu^\dis \to 0$. This does not indicate that privacy leakage is amplified as the owner becomes honest, as shown in Fig. \ref{fig:pareto-main} and \ref{sfig:eta-0.1}-\ref{sfig:eta-0.9}, but that the privacy displacement is fixed while the compliance signal vector itself goes to $0$.

In Fig. \ref{sfig:gamma-distribution}, we observe that the privacy leakage is not uniform across $D_r$: while a majority of $z^*$ are non-aligned with $D_u$ with $\gamma_Q(z^*) \approx 0$, a few $z^*$ (e.g., at the 90\% or 99\% percentile) allow compliance queries to more efficiently contribute to curiosity advantages. Therefore, the honest-but-curious $\Aud$ presents an elevated privacy threat to specific target samples. This is consistent with findings for MIAs against both ML \cite{Carlini2022Onion} and MU \cite{Hayes2024Inexact} algorithms: privacy inference under MU is instance-sensitive, where a few samples have significantly higher privacy leakage compared to the population $D_r$.

More interestingly, we study how different query sequences $Q$ affects the audit beyond the Pareto points in Fig. \ref{fig:pareto-main} and \ref{fig:pareto-convex}. As shown in Fig. \ref{sfig:gamma-Q-sel}, different selections of $Q$ induce different $\gamma_Q$, which can yield audit-visible alignment far above/below the query-independent baseline $\kappa$. The proxy $\lambda_Q^\conv$ is highly correlated with $\gamma_Q$. For different unlearned sets $D_u$, $\gamma_Q$ decreases as $|D_u|/|D|$ increases, implying that the privacy of individual target samples is better protected in the audit process by a more generalized unlearning set. Nevertheless, we observe that $\gamma_Q$ degrades slightly slower for top 99\% percentile targets; therefore, a subset of the population remains vulnerable to privacy leakage even
with a large unlearning set.

\section{Discussion and Future Work}\label{sec:dis}



\subsection{From Membership Inference to General Privacy Inferences}\label{ssec:dis-priv-inf}

In Theorem \ref{th:main}, we design the curiosity game $\Cur^\Pi$ (Game \ref{game:cur}) for audit transcript-based membership inference ($\mathsf{MI}$) of whether target $z^*$ is a member of the retained set $D_r$. However, we note that Theorem \ref{th:main} is not intrinsically connected to membership status, but rather to arbitrary \emph{binary} privacy-sensitive statistics of $D_r$. For any binary hypothesis $s(D_r)\in\{0,1\}$ that induces endpoint distributions $\thu^{s=1}$ and $\thu^{s=0}$, let:
\begin{equation}
\begin{aligned}
    p_Q(s(D_r)) &:= F_Q(\thu^{s=1})-F_Q(\thu^{s=0}),\\
    \gamma_Q(s(D_r)) &:= \frac{|\langle p_Q(s),c_Q\rangle|}{\|c_Q\|_2^2}.
\end{aligned}
\end{equation}
Theorem \ref{th:main} therefore applies verbatim to the corresponding binary distinguishability whenever we assume the equal-covariance Gaussian observation noise transcript (Assumption \ref{assum:audit}).

To further extend our framework to an honest-but-curious auditor $\Aud$ that aims to infer additional non-binary privacy-sensitive properties of $D_r$ similar to other privacy inference attacks, e.g., attribute inference ($\mathsf{AI}$), data reconstruction ($\mathsf{RC}$), and property inference ($\mathsf{PI}$), Salem et al. \cite{Salem2023PrivGames} establishes how some privacy inference games can be translated, and Kulynych et al. \cite{Kulynych205Unifying} provide a unified $f$-DP bound for the curiosity advantage of $\mathsf{MI}$, $\mathsf{AI}$ and $\mathsf{RC}$. We leave for future work to (\rom{1}) deriving audit-specific reduction constants, and (\rom{2}) investigate how property inference, which concerns data distribution-level rather than record-level hypothesis testing.

\subsection{Towards a Foundation for Secure and Privacy-Preserving Audit of MU}\label{ssec:dis-foundation}

Our results identify the behavioral geometry that determines when accurate auditing of MU forces privacy leakage on the retained set $D_r$, quantified exactly by $(c_Q, p_Q, \gamma_Q)$ as the compliance/curiosity vectors and the transfer coefficient; or approximately by $(\lambda_Q,U_Q)$.
This characterization allows us to investigate complementary approaches to reduce the privacy leakage through behavioral MU audit.
Beyond the theoretical and empirical confirmation, we also ask the question:
\begin{quote}
    How should we design a privacy-preserving audit of Machine Unlearning that addresses the tension between successful audit (enforcement of $\unA$) and privacy leakage (protection of $D_r$)?
\end{quote}
We present three natural mechanisms stemming from our discussion:

\paragraph{Differential Privacy} Naturally, injecting privacy perturbation to the audit transcript $\tau$ limits the contribution of each query $x_t$ and therefore bounds the auditor's instance-level curiosity advantage $\beta(z^*)$ \cite{Dwork2006DP}. The role of DP is delicate and can be categorized into three stages of the MU pipeline:
\begin{enumerate}[label=(\roman*)]
    \item \textbf{Training Stage}, which enforces DP during the initial training process. Intuitively, for an $\varepsilon$- or $(\varepsilon,\delta)$-differentially private model $\theta$, any \emph{behavioral} audit transcript (Assumption \ref{assum:audit}) becomes post-processing.
    
    DP directly constrains the privacy displacement $p_Q$. However, we note that the compliance distinguishability that compares $\thu^\hon$ with $\thu^\non$ are conditioned on $\unA$ and $\unA^\dis$, rather than adjacent training sets. Therefore, DP does not directly affect $c_Q$, and requires careful design of the mechanism to not dampen the behavioral signal separation needed for auditing accuracy $1-\alpha-\alpha^\prime$.
    
    \item \textbf{Unlearning Stage}, which enforces DP during the unlearning process. Notably, the family of $(\varepsilon, \delta)$-certified unlearning algorithms we employed (Def. \ref{def:MU} and Assumption \ref{assum:convex}) does not offer privacy guarantees, but rather offer \emph{sufficiency} guarantees that $\thu$ is indistinguishable from retraining \cite{Guo2020Cert,Koloskova2025Cert}. A similar tension to training-stage DP persists: enforcing DP at the unlearning stage alone only controls $p_Q$ without accounting for $c_Q$.

    \item \textbf{Output Stage}, similar to output filtering (Sec. \ref{ssec:lit-mu}), the model owner $\Chal$ can release noisy posteriors of model output.
    Importantly, we note that the role of DP noise is \emph{distinct from} the Gaussian observation noise in Assumption \ref{assum:audit}. Uncalibrated noise only rescales the $(\alpha+\alpha^\prime, \beta)$-tradeoff alone the Pareto front (Fig. \ref{sfig:eta-0.1}-\ref{sfig:eta-0.9}) and does not change the underlying alignment coefficient $\gamma_Q$, and can be mitigated by a higher query budget $T$;
    
    An $\varepsilon$- or $(\varepsilon,\delta)$-DP output mechanism requires its noises to be calibrated to $D_r$-record sensitivity, and also account for composition across $T$ queries. Consequently, output DP can bound curiosity leakage $p_Q$ but at the cost of constraining $c_Q$ at the same time.
\end{enumerate}

\paragraph{Query Restriction} One route to bounding privacy leakage is to revoke $\Aud$'s oracle access (see Sec. \ref{ssec:threat-model}) and instead restrict $\Aud$ to submitting only compliance-relevant queries, which directly changes the audit-visible geometry: changing $Q$ affects changes $(c_Q, p_Q, \gamma_Q)$ through $M_Q$ on top of the query-independent geometry, and can significantly affect query efficacy (Fig. \ref{sfig:gamma-Q-sel}, Sec. \ref{ssec:result-convex}).
A restricted $\Pi$ can therefore only allow queries that preserve compliance separation while reducing alignment with likely privacy-sensitive directions. However, as discussed in Remark \ref{rem:genericity}, alignment reduction in itself removes the component of privacy leakage forced by $c_Q$, but an orthogonal component \(p_Q^\perp\) remains. Moreover, designing $Q$ against the exact privacy directions would generally require information about any $z\in D_r$ that allows $\Aud$ to trivially win the curiosity game (Game \ref{game:cur}).

\paragraph{Cryptographic Proof-of-Unlearning} Proof-of-Unlearning schemes (Sec. \ref{ssec:lit-audit}) enforce a much stronger limit on the ability of $\Aud$ than query restriction, and limit the auditor's access to only \emph{verification} of information released by $\Chal$. Therefore, such an audit protocol $\Pi$ releases less information than the model behavior $F_Q(\theta_u)$. Under our framework, this approach avoids querying a behavioral channel that exposes the compliance/curiosity geometry.
However, as discussed by Zhang et al. in \cite{Zhang2024Fragile}, proof-of-unlearning schemes are themselves forgeable by dishonest model owners that bypass the protocol. Further, releasing gradient- or parameter-level information to third-party auditors enables $\Aud$ to reconstruct the training samples \cite{Wen2025SoK,Haim2022Reconstruct}.

In conclusion, the three mechanisms presented affect the geometry $(c_Q, p_Q, \gamma_Q)$ directly and indirectly at different layers of the MU pipeline, but each poses a weakness that it does not address on its own.
We argue that no single mechanism is sufficient to achieve a secure, privacy-preserving MU, and can require a combination of DP at learning, DP-aware unlearning \cite{Sekhari2021Remember}, and cryptographic proof-of-unlearning that ensures execution of the MU mechanism. A privacy-preserving audit protocol $\Pi$ that is simultaneously sound against dishonest execution, non-leaking on $D_r$, and resilient to forgery remains an open research question.

\section{Conclusion}\label{sec:conclusion}

\paragraph{Contributions} In this paper, we prove a fundamental privacy-audit alignment for Machine Unlearning (MU), in which an honest-but-curious auditor that follows a purely behavioral audit protocol incurs membership privacy leakage on the retained set characterized by a transfer coefficient. Our theoretical results are confirmed on membership inference experiments on MLP/CNNs and controlled experiments on convex logistic regression models, and demonstrate how the audit components affect privacy leakage. Our theoretical results are confirmed on membership inference experiments on MLP/CNNs and controlled experiment on convex logistic regression models, and demonstrates how the audit components affect privacy leakage. Overall, our results call for a more careful consideration of the privacy-audit tension for MU, and serve as a foundation for more scrutiny of designs of privacy-preserving audit schemes for MU.

\paragraph{Limitations} We acknowledge several limitations in our paper. Our analysis focuses on fixed-query behavioral audits and derives its sharpest characterization under an equal-covariance Gaussian observation noise (Assumption \ref{assum:audit}); more general \emph{adaptive} protocols and \emph{sub-Gaussian} noise models remain outside of our scope. Second, we assume a capable auditor in Sec. \ref{sec:eval} that has access to the underlying data distribution and knowledge of the ML/MU pipeline, while less auditor knowledge can shift the approximate geometry certificate $(\lambda_Q, U_Q)$ and affect the exposed privacy leakage. Empirically, we evaluate our theoretical results on representative MLP/CNN models and a controlled convex setting that may not be representative of more MU models/algorithms. We note that these limitations affect how the audit geometry and privacy leakage can be estimated \emph{in practice}, while the theoretical privacy-audit transfer result explicitly assumes no model/algorithm restrictions.

\section*{Acknowledgments}
This work is supported by Cisco Research and the National Science Foundation under Grant Number 2624954.

\bibliographystyle{plainurl}
\bibliography{IEEEabrv, bib}

\appendix
\section*{Ethical Considerations}
In this paper, we provide theoretical and empirical evidence for privacy-audit alignment in auditing Machine Unlearning mechanisms by showing that a purely behavioral audit scheme imposes membership privacy leakage on the retained data forced by the audit accuracy. We study how to prevent an honest-but-curious auditor from inferring privacy-sensitive information of the audit subject.
While our result can be adopted by the honest-but-curious auditors we described, it also serves as evidence for refuting such MU audit schemes that contain potential privacy risk. Therefore, we believe the merit and contribution of this work far exceed the potential risk.

\section*{Open Science}
To comply with open science, we include all code used for the experiments in \href{https://anonymous.4open.science/r/Behavioral-Unlearn-Audit-2F68/README.md}{this repository}, which includes instructions on executing the code. All datasets we use are publicly available. We do not include trained/unlearned model binaries as their quantities and individual sizes are prohibitive; they can be trained/unlearned from scratch with the provided code.

\section{Proof to Lemma \ref{lemma:cert}}\label{app:proof:cert}
\begin{proof}
    We can decompose $\tilde{p}_Q$ as:
    \begin{equation}
        \tilde{p}_Q = a_Q\tilde{c}_Q + r_Q, \quad
        a_Q := \frac{\langle\tilde{p}_Q,\tilde{c}_Q\rangle}{\|\tilde{c}_Q\|_2^2}, \quad
        r_Q \perp \tilde{c}_Q.
    \end{equation}
    Naturally, $\|\tilde{p}_Q\|_2 \geq |a_Q|\|\tilde{c}_Q\|_2 = \lambda_Q\|\tilde{c}_Q\|_2$. Therefore:
    \begin{equation}
    \begin{aligned}
        \|p_Q\|_2 &\geq \|\tilde{p}_Q\|_2-\varepsilon_p \\
        &\geq \lambda_Q\|\tilde{c}_Q\|_2-\varepsilon_p \\
        &\geq \lambda_Q(\|c_Q\|_2-\varepsilon_c)-\varepsilon_p,
    \end{aligned}
    \end{equation}
    Recall $\|p_Q\|_2 \geq 0$, which yields:
    \begin{equation}\label{eq:geo-transfer}
        \|p_Q(z^*)\|_2 \geq \big[\lambda_Q\|c_Q\|_2-U_Q\big]_+,
    \end{equation}
    in which we use the notation $[u]_+:=\max\{u,0\}$. Similar to Theorem \ref{th:main}, the audit accuracy guarantee implies that $\|c_Q\|_2 \geq 2\sigma q_e$, therefore:
    \begin{equation}
        \|p_Q\|_2 \geq [2\sigma\lambda_Qq_e-U_Q]_+
    \end{equation}
    Substituting this into Eq. \ref{eq:alpha-beta-d} yields Eq. \ref{eq:main-exact-cert}. Further, follow the same concavity argument in the proof to Theorem \ref{th:main}, recall that $g(u)=2\Phi(u)-1$ is $\sqrt{2\pi}$-Lipschitz, yields Eq. \ref{eq:main-linear-cert}.
\end{proof}

\section{Proof to Corollary \ref{cor:loc-aud-geo}}\label{app:proof:loc-aud-geo}
\begin{proof}
    For any $(\theta, \theta^\prime)$ in the stated neighborhood, we have:
    \begin{equation}
        \|J_Q(\theta)-J_Q(\theta^\prime)\|_\mathrm{op} \leq L_Q\|\theta-\theta^\prime\|_2.
    \end{equation}
    Therefore, for any $\vartheta$:
    \begin{equation}
        \|F_Q(\vartheta)-F_Q(\theta_0)-J_Q(\vartheta-\theta_0)\|_2 \leq
        \frac{L_Q}{2}\|\vartheta-\theta_0\|_2^2.
    \end{equation}
    Subtract the two expansions for $\thu^\hon$ and $\thu^\dis$. By triangle inequality, we have:
    \begin{equation}
        \|c_Q-J_Q\delta_c\|_2\leq\varepsilon_c^\loc.
    \end{equation}
    The same applies to the member/non-member endpoint pair $(\thu^\mem, \thu^\non)$. Recall:
    \begin{equation}
        \tilde{c}_Q=J_Q\delta_c, \quad \tilde{p}_Q=J_Q\delta_p,
    \end{equation}
    by Def. \ref{def:geo-cert}, we have
    \begin{equation}
        \lambda_Q^\loc
        =\frac{|\delta_p^\top J_Q^\top J_Q\delta_c|}
        {\delta_c^\top J_Q^\top J_Q\delta_c},
    \end{equation}
    and:
    \begin{equation}
        U_Q=\lambda_Q^\loc\varepsilon_c^\loc+\varepsilon_p^\loc.
    \end{equation}
\end{proof}

\section{Proof to Corollary \ref{cor:conv-aud-geo}}\label{app:proof:conv-aud-geo}
\begin{proof}
    By definition, the compliance endpoints are:
    \begin{equation}
        \thu^\hon=\theta+\Delta, \quad \thu^\dis=\theta+\eta\Delta
    \end{equation}
    Taylor expansion on both endpoints around $\theta$ gives:
    \begin{equation}\label{eq:taylor-comp-conv}
        \|c_Q-J_Q\delta_c\|_2 \leq \frac{L_Q}{2}(1+\eta^2) \|\Delta\|_2^2.
    \end{equation}
    For the curiosity endpoints, write:
    \begin{equation}
        \delta_p=\delta_p^0+r_\theta.
    \end{equation}
    Similar to Eq. \ref{eq:taylor-comp-conv}, Taylor expand on $\thu^\mem$ and $\thu^\non$; Recall:
    \begin{equation}
        \|J_Qr_\theta\|_2 \leq \|J_Q\|_{\mathrm{op}}R_\theta,
    \end{equation}
    therefore:
    \begin{equation}
        \varepsilon_p^\conv = \|J_Q\|_{\mathrm{op}}R_\theta +
        \frac{L_Q}{2}\big( \|\thu^\mem-\theta\|_2^2+\|\thu^\non-\theta\|_2^2 \big)
    \end{equation}
    Eq. \ref{eq:lambda-conv} then gives the certificate coefficient; Eq. \ref{eq:convex-U} follows from Def. \ref{def:geo-cert}, and the two privacy bounds follow directly from Lemma \ref{lemma:cert}.
\end{proof}

\section{Additional Experiments for Sec. \ref{sec:eval}}\label{app:add-exp}

\paragraph{Pareto Curve} We include additional experiments of MLP on Purchase-100 \cite{Shokri2017MIA} and ResNet-18 \cite{He2016ResNet} on CIFAR-10 \cite{Krizhevsky2009CIFAR} here.

\begin{figure*}[t]
    \centering
    \captionsetup[subfigure]{justification=centering}
    \begin{minipage}[t]{0.75\linewidth}
    \begin{tcolorbox}[groupbox, title=MLP @ Purchase-100 \cite{Shokri2017MIA}]
        \subfloat[\textbf{GA}, $\eta = 0.1$.]
        {\includegraphics[width=.33\linewidth]{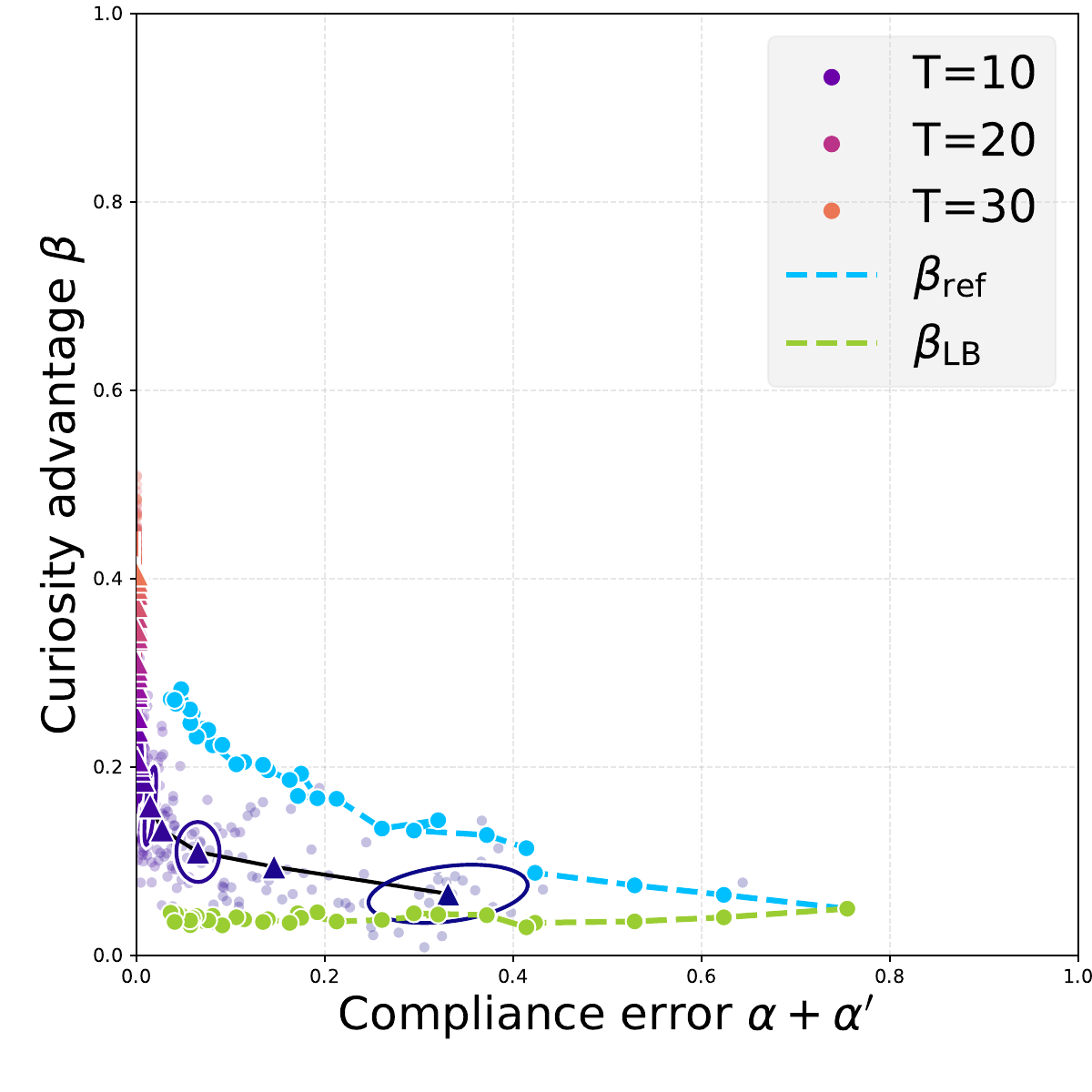}}
        \subfloat[\textbf{GA}, $\eta = 0.5$.]
        {\includegraphics[width=.33\linewidth]{Figs/p100-mlp4-rand/GradAscent_0.5.pdf}}        
        \subfloat[\textbf{GA}, $\eta = 0.9$.]
        {\includegraphics[width=.33\linewidth]{Figs/p100-mlp4-rand/GradAscent_0.9.pdf}}
        
        \subfloat[\textbf{SalUn}, $\eta = 0.1$.]
        {\includegraphics[width=.33\linewidth]{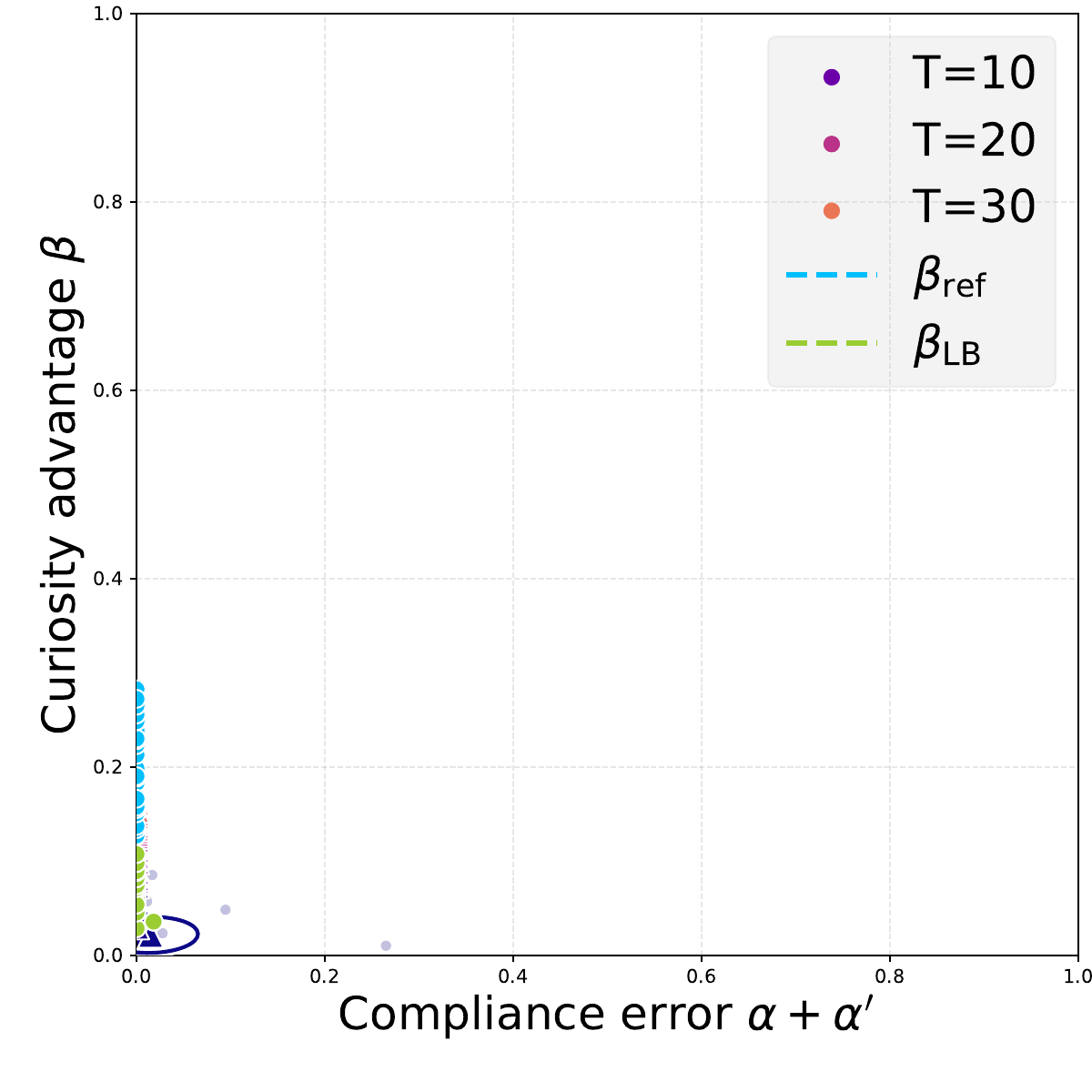}}
        \subfloat[\textbf{SalUn}, $\eta = 0.5$.]
        {\includegraphics[width=.33\linewidth]{Figs/p100-mlp4-rand/SalUn_0.5.pdf}}        
        \subfloat[\textbf{SalUn}, $\eta = 0.9$.]
        {\includegraphics[width=.33\linewidth]{Figs/p100-mlp4-rand/SalUn_0.9.pdf}}
    \end{tcolorbox}
    \end{minipage}
    \caption{Pareto curves for MLPs on Purchase-100, random-sampled unlearned set, across \textbf{GA} and \textbf{SalUn}, $\eta=0.1,0.5,0.9$.}
    \label{fig:pareto-app-p100-mlp4-rand}
\end{figure*}

\begin{figure*}[t]
    \centering
    \captionsetup[subfigure]{justification=centering}
    \begin{minipage}[t]{0.75\linewidth}
    \begin{tcolorbox}[groupbox, title=ResNet-18 \cite{He2016ResNet} @ CIFAR-10 \cite{Krizhevsky2009CIFAR} (Random-Sampled $D_u$)]
        \subfloat[\textbf{GA}, $\eta = 0.1$.]
        {\includegraphics[width=.33\linewidth]{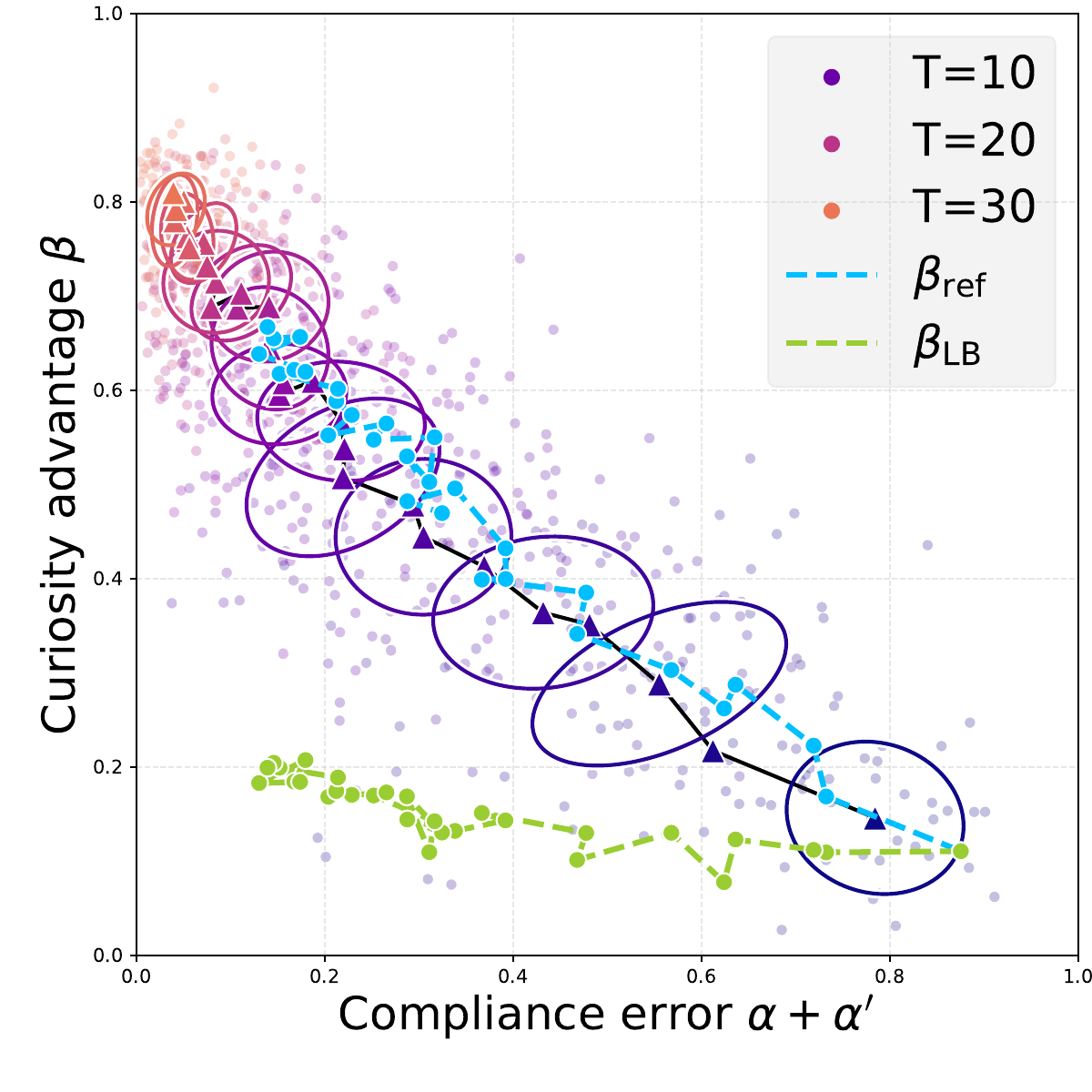}}
        \subfloat[\textbf{GA}, $\eta = 0.5$.]
        {\includegraphics[width=.33\linewidth]{Figs/c10-r18-rand/GradAscent_0.5.pdf}}        
        \subfloat[\textbf{GA}, $\eta = 0.9$.]
        {\includegraphics[width=.33\linewidth]{Figs/c10-r18-rand/GradAscent_0.9.pdf}}
        
        \subfloat[\textbf{SalUn}, $\eta = 0.1$.]
        {\includegraphics[width=.33\linewidth]{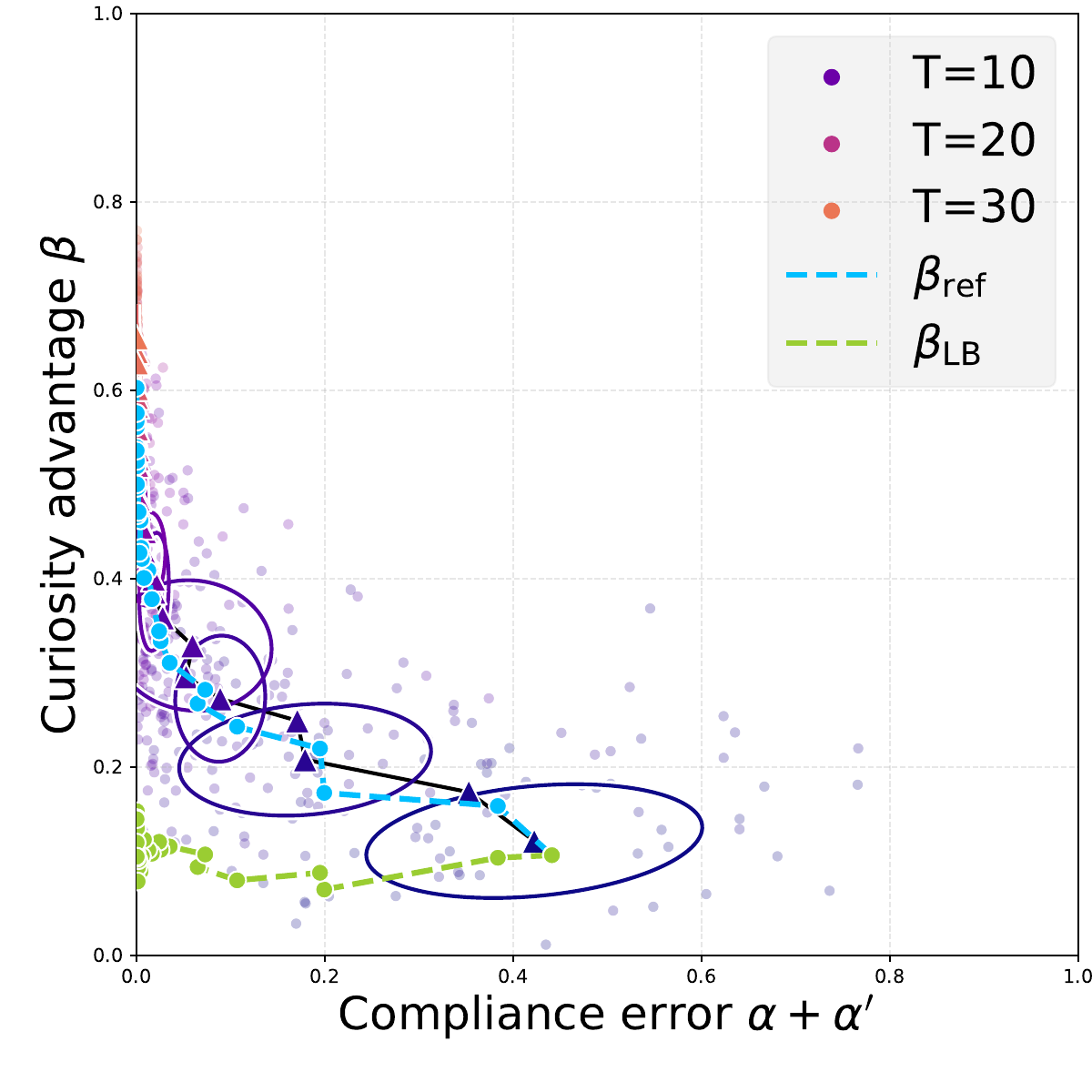}}
        \subfloat[\textbf{SalUn}, $\eta = 0.5$.]
        {\includegraphics[width=.33\linewidth]{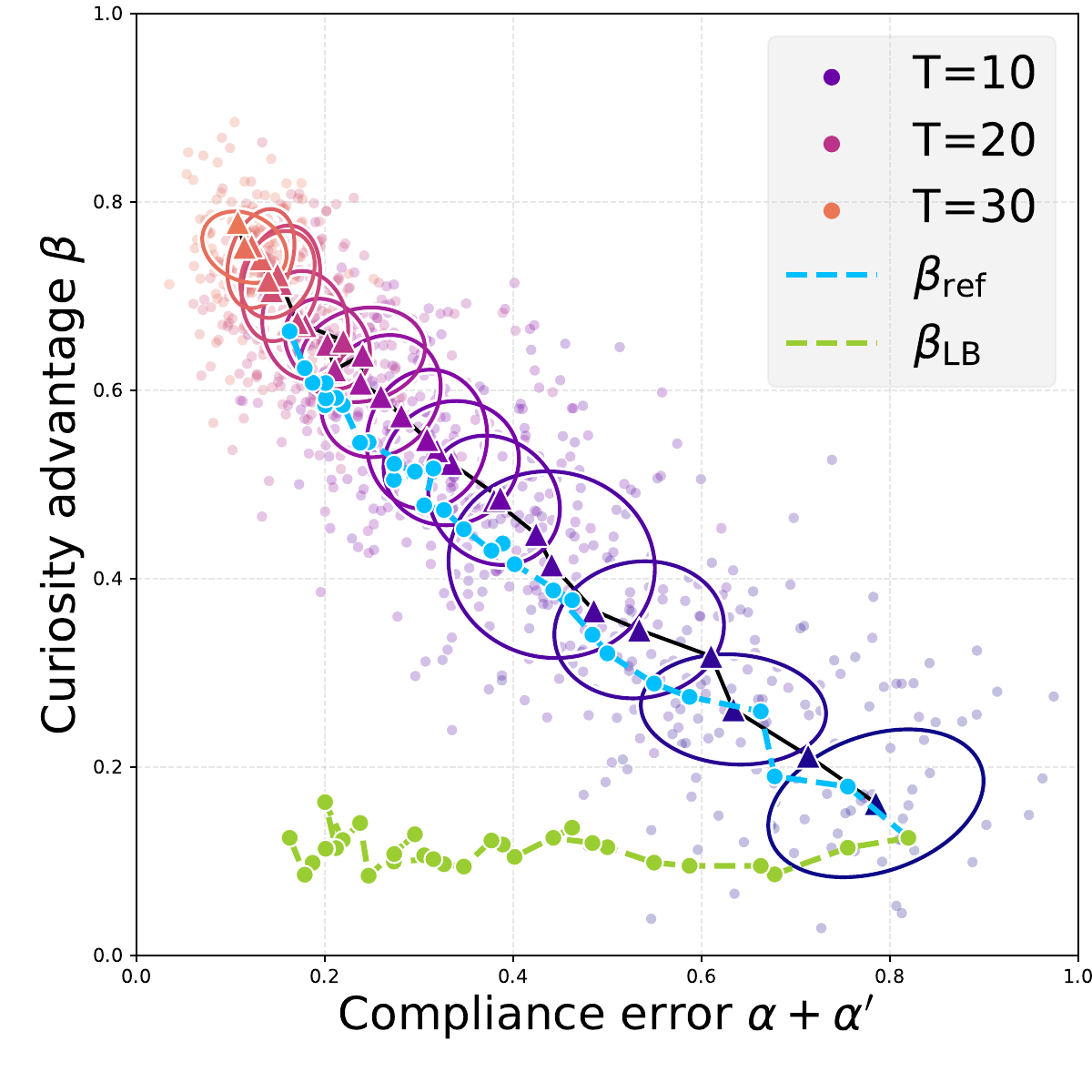}}        
        \subfloat[\textbf{SalUn}, $\eta = 0.9$.]
        {\includegraphics[width=.33\linewidth]{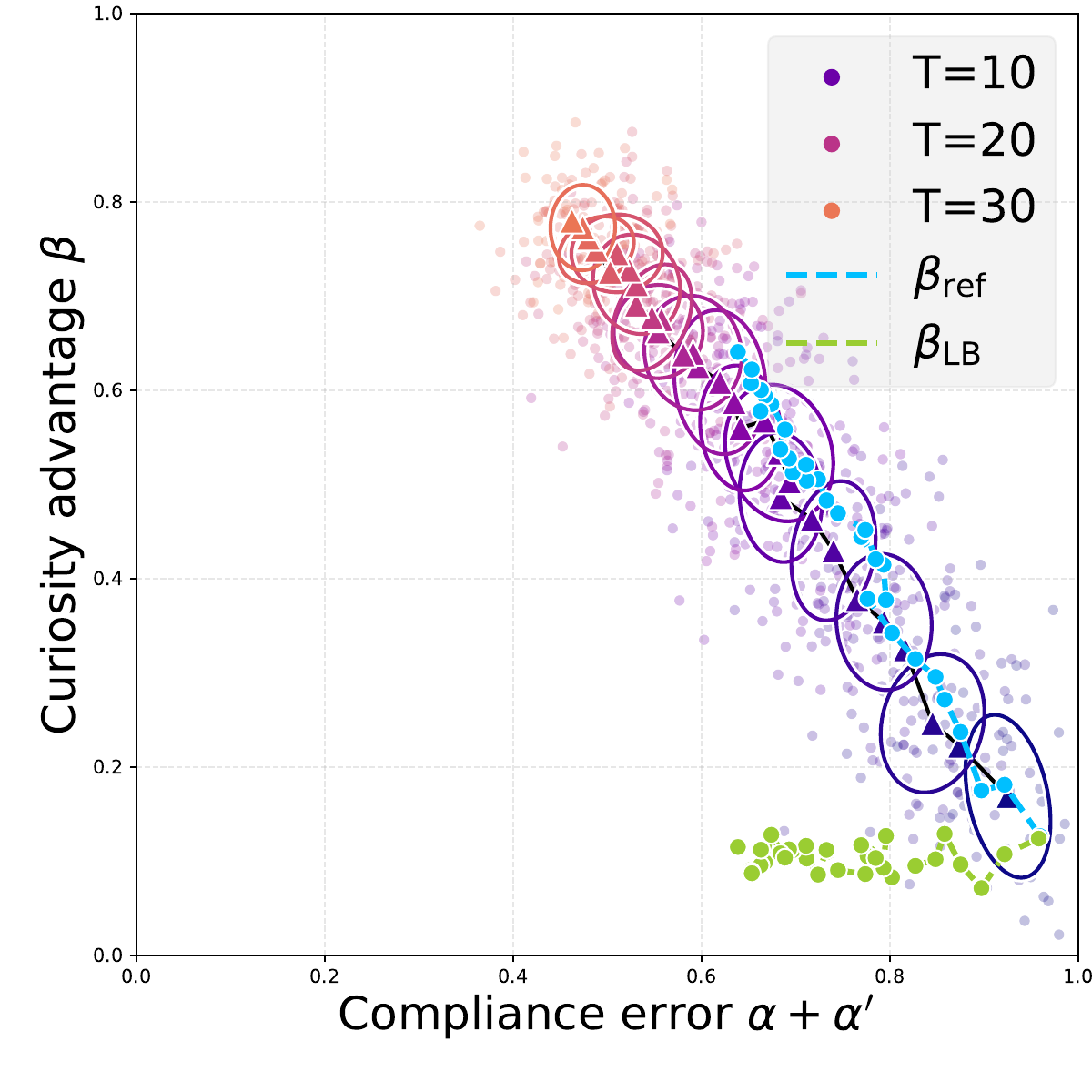}}
    \end{tcolorbox}
    \end{minipage}
    \caption{Pareto curves for MLPs on Purchase-100, random-sampled unlearned set, across \textbf{GA} and \textbf{SalUn}, $\eta=0.1,0.5,0.9$.}
    \label{fig:pareto-app-c10-r18-rand}
\end{figure*}

\begin{figure*}[t]
    \centering
    \captionsetup[subfigure]{justification=centering}
    \begin{minipage}[t]{0.75\linewidth}
    \begin{tcolorbox}[groupbox, title=ResNet-18 \cite{He2016ResNet} @ Purchase-100 \cite{Krizhevsky2009CIFAR} (Class-wise $D_u$)]
        \subfloat[\textbf{GA}, $\eta = 0.1$.]
        {\includegraphics[width=.33\linewidth]{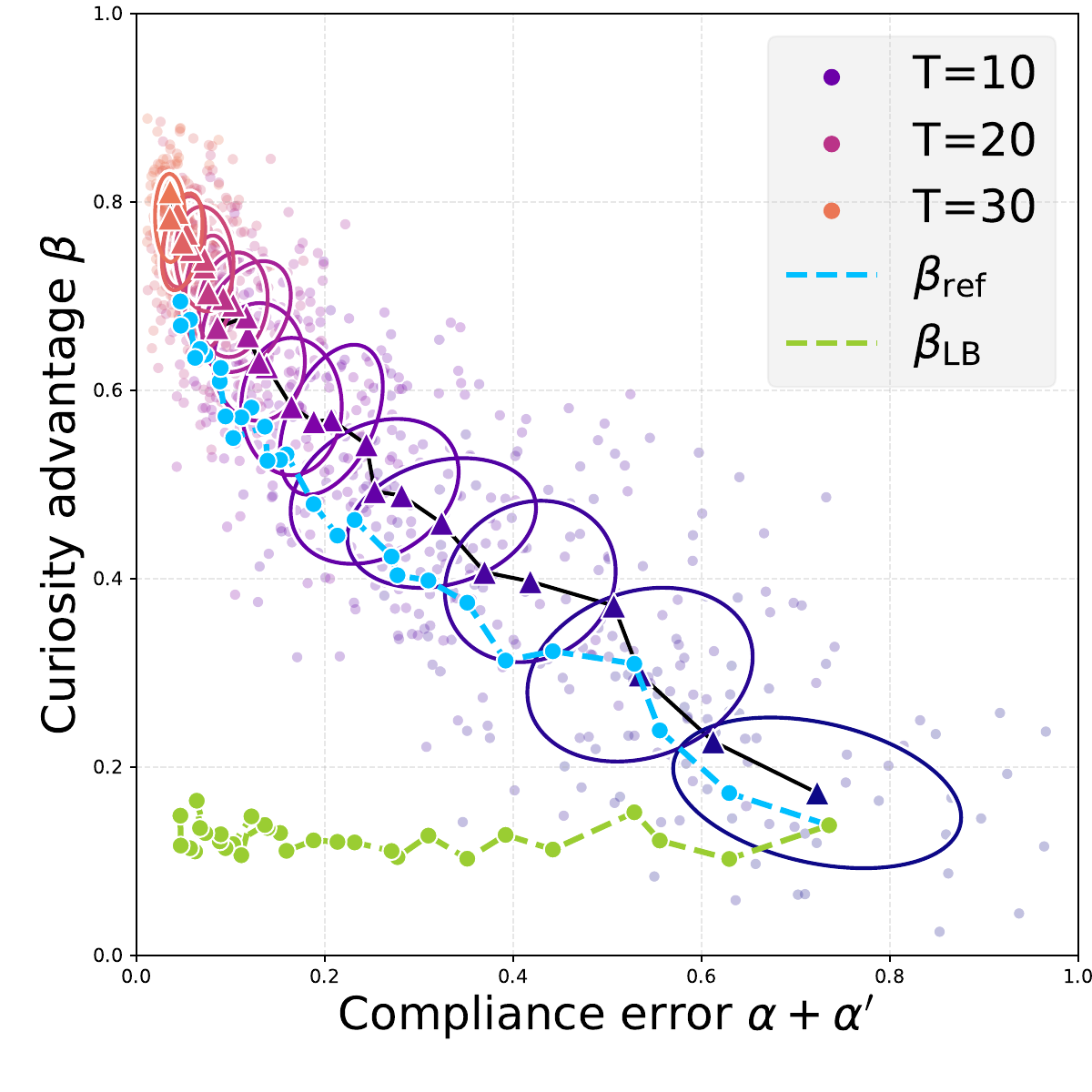}}
        \subfloat[\textbf{GA}, $\eta = 0.5$.]
        {\includegraphics[width=.33\linewidth]{Figs/c10-r18-class/GradAscent_0.5.pdf}}        
        \subfloat[\textbf{GA}, $\eta = 0.9$.]
        {\includegraphics[width=.33\linewidth]{Figs/c10-r18-class/GradAscent_0.9.pdf}}
        
        \subfloat[\textbf{SalUn}, $\eta = 0.1$.]
        {\includegraphics[width=.33\linewidth]{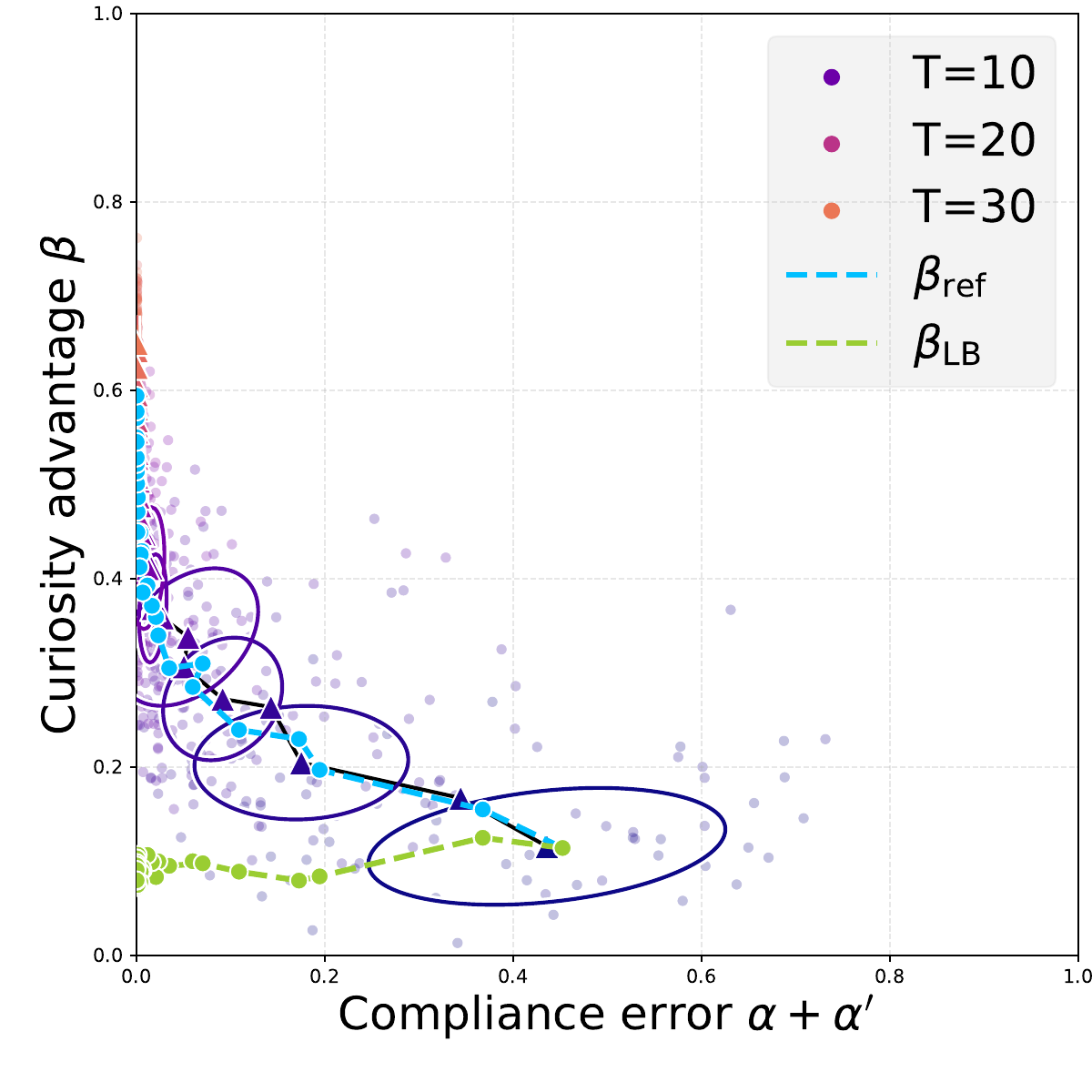}}
        \subfloat[\textbf{SalUn}, $\eta = 0.5$.]
        {\includegraphics[width=.33\linewidth]{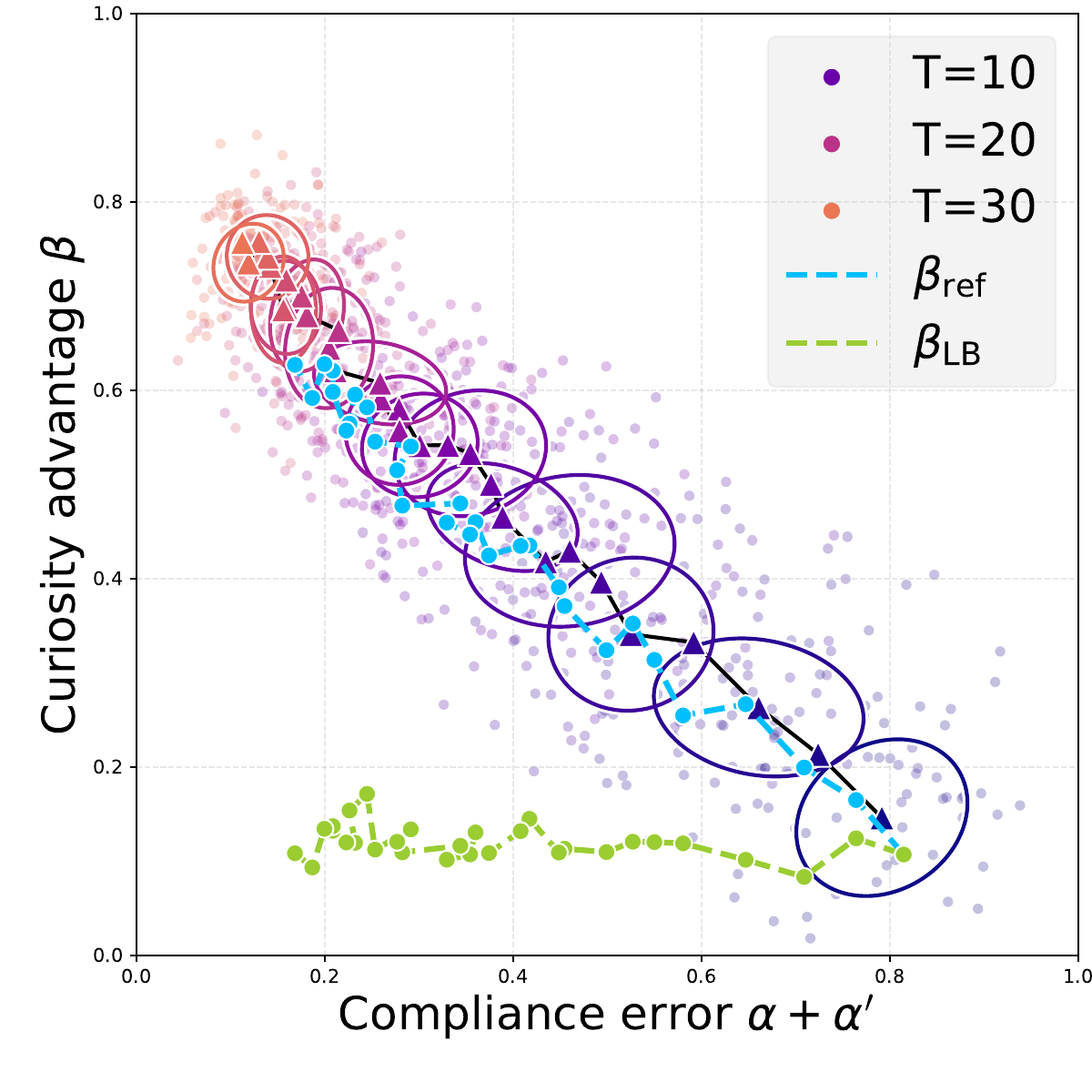}}        
        \subfloat[\textbf{SalUn}, $\eta = 0.9$.]
        {\includegraphics[width=.33\linewidth]{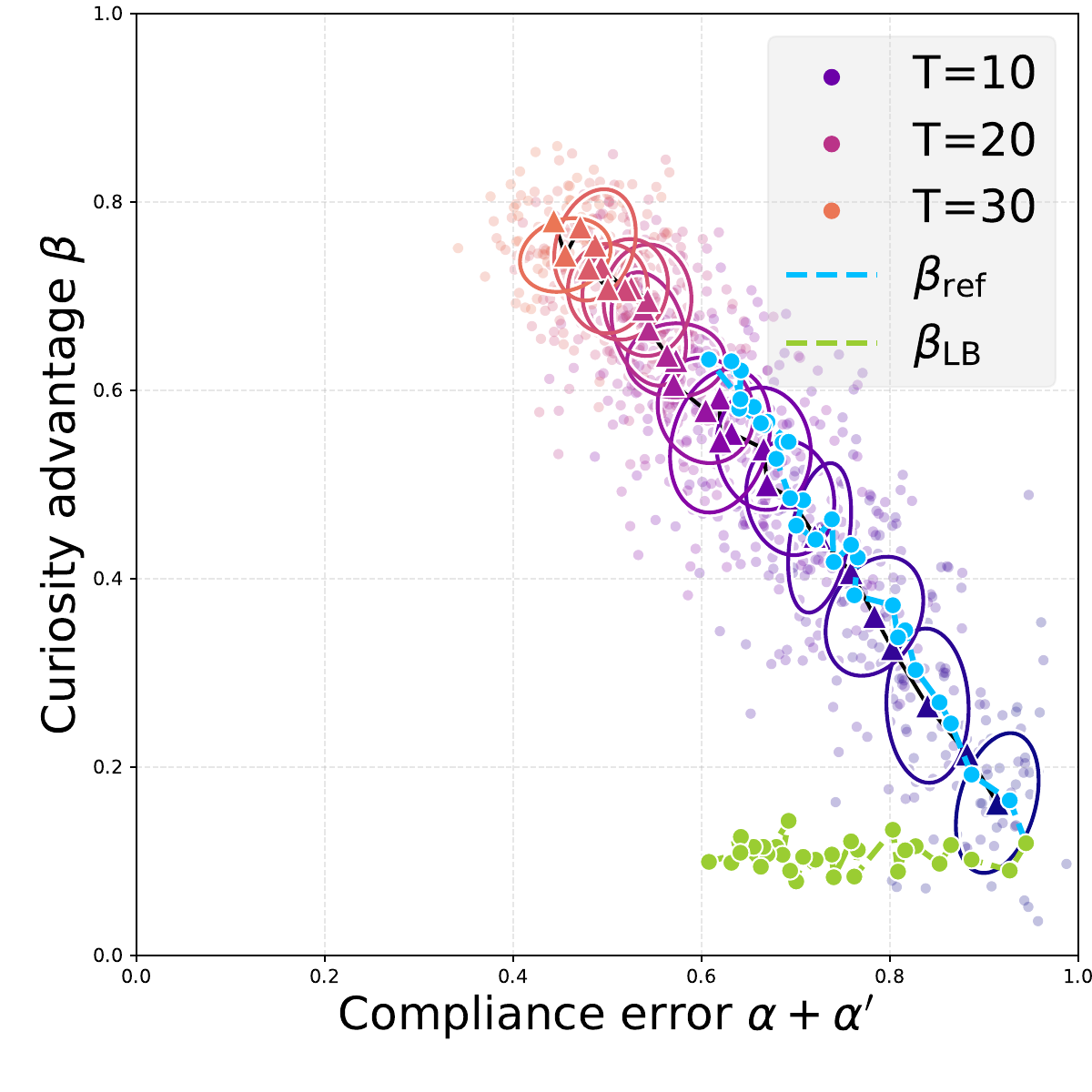}}
    \end{tcolorbox}
    \end{minipage}
    \caption{Pareto curves for MLPs on Purchase-100, class-wise unlearned set, across \textbf{GA} and \textbf{SalUn}, $\eta=0.1,0.5,0.9$.}
    \label{fig:pareto-app-c10-r18-class}
\end{figure*}

\end{document}